\documentclass[10pt,journal]{IEEEtran}
\usepackage{graphicx}
\usepackage{epsfig}
\usepackage{amsmath,amssymb,mathrsfs} 
\usepackage{amsthm}
\usepackage{color}
\usepackage{url}
\usepackage{array}
\usepackage{subfigure}
\usepackage[vlined,boxed,commentsnumbered]{algorithm2e}

\def\sint{\begingroup\textstyle\int\endgroup} 
\graphicspath{{./fig/}}
\newtheorem{theorem}{Theorem}
\newtheorem{corollary}[theorem]{Corollary}
\newtheorem{lemma}[theorem]{Lemma}

\DeclareMathOperator*{\sparse}{Sp}
\DeclareMathOperator*{\Log}{Log}
\DeclareMathOperator*{\Exp}{Exp}
\DeclareMathOperator*{\genlinear}{GL}
\DeclareMathOperator*{\argmin}{argmin}
\DeclareMathOperator*{\trace}{Tr}

\DeclareMathOperator*{\expmap}{exp} 
\DeclareMathOperator*{\logmap}{log}

\newcommand{\spd}[1]{\mathcal{S}^{#1}_{+}}

\newcommand{\atom}{B}
\newcommand{\datapt}{X}
\newcommand{\dataZ}{S}
\newcommand{\dataset}{\mathfrak{X}}
\newcommand{\half}{\frac{1}{2}}

\newcommand{\norm}[2]{\left\|{#1}\right\|_{#2}}
\newcommand{\fnorm}[1]{\norm{#1}{\text{F}}}
\newcommand{\enorm}[1]{\norm{#1}{2}}

\newcommand{\riem}{d_{\mathcal{R}}}

\newcommand{\dict}{\mathbf{B}}

\newcommand{\set}[1]{\left\{{#1}\right\}}
\newcommand{\gdatapt}{z}
\newcommand{\loss}{\mathcal{L}}
\newcommand{\lerm}{d_{\text{le}}}
\newcommand{\sym}[1]{\mathcal{S}^{#1}}
\newcommand{\logdet}[1]{\log\left|{#1}\right|}
\newcommand{\stein}{d_{\text{S}}}
\newcommand{\burg}{d_{\text{B}}}

\newcommand{\eye}{I}

\newcommand{\reals}[1]{\mathbb{R}^{#1}}
\newcommand{\nlmin}{\min\nolimits}
\newcommand{\nlsum}{\sum\nolimits}
\newcommand{\ip}[2]{\left\langle {#1},\, {#2}\right\rangle}

\newcommand{\proj}[1]{\mathscr{P}[#1]}
\newcommand{\semipd}[1]{\mathcal{S}^{#1}}

\newcommand{\manifold}{\mathcal{M}}
\newcommand{\tangent}[1]{T_{#1}}
\newcommand{\tangsp}{\tangent{P}\manifold}

\DeclareMathOperator*{\grad}{grad}
\newcommand{\vectrans}[2]{\mathfrak{T}_{#1}(#2)}

\DeclareMathOperator*{\abs}{abs}
\renewcommand{\textcolor}[1]{}

\begin{document}
%
\title{Riemannian Dictionary Learning and Sparse Coding for Positive Definite Matrices}
\author{Anoop~Cherian\quad\quad Suvrit~Sra
\thanks{Anoop Cherian is with the ARC Centre of Excellence for Robotic Vision, Australian National University, Canberra, Australia\quad Email: anoop.cherian@anu.edu.au.\protect\\}
\thanks{Suvrit Sra is with the Massachusetts Institute of Technology.\protect\\
E-mail: suvrit@mit.edu}
}



\maketitle
\begin{abstract}
Data encoded as symmetric positive definite (SPD) matrices frequently arise in many areas of computer vision and machine learning. While these matrices form an open subset of the Euclidean space of symmetric matrices, viewing them through the lens of non-Euclidean Riemannian geometry often turns out to be better suited in capturing several desirable data properties. However, formulating classical machine learning algorithms within such a geometry is often non-trivial and computationally expensive. Inspired by the great success of dictionary learning and sparse coding for vector-valued data, our goal in this paper is to represent data in the form of SPD matrices as sparse conic combinations of SPD atoms from a learned dictionary via a Riemannian geometric approach. To that end, we formulate a novel Riemannian optimization objective for dictionary learning and sparse coding in which the representation loss is characterized via the affine invariant Riemannian metric. We also present a computationally simple algorithm for optimizing our model. Experiments on several computer vision datasets demonstrate superior classification and retrieval performance using our approach when compared to sparse coding via alternative non-Riemannian formulations. 

\end{abstract}


\IEEEpeerreviewmaketitle



\ifCLASSOPTIONcaptionsoff
  \newpage
\fi

\section{Introduction}
\label{sec:intro}
\IEEEPARstart{S}{ymmetric} positive definite (SPD) matrices play an important role as data descriptors in several computer vision applications, for example in the form of region covariances~\cite{porikli1}. Notable examples where such descriptors are used include object recognition~\cite{cvpr_hartley}, human detection and tracking~\cite{porikli3}, visual surveillance~\cite{cherian2011dirichlet}, 3D object recognition~\cite{fehr2012compact}, among others. Compared with popular vector space descriptors, such as bag-of-words, Fisher vectors, etc., the second-order structure offered by covariance matrices is particularly appealing. For instance, covariances conveniently fuse multiple features into a compact form independent of the number of data points. By choosing appropriate features, this fusion can be made invariant to affine distortions~\cite{ma2012affine}, or robust to static image noise and illumination variations. Moreover, generating these descriptors is easy, for instance using integral image transforms~\cite{porikli3}.

\begin{figure}[ht]
	\centering	
    \includegraphics{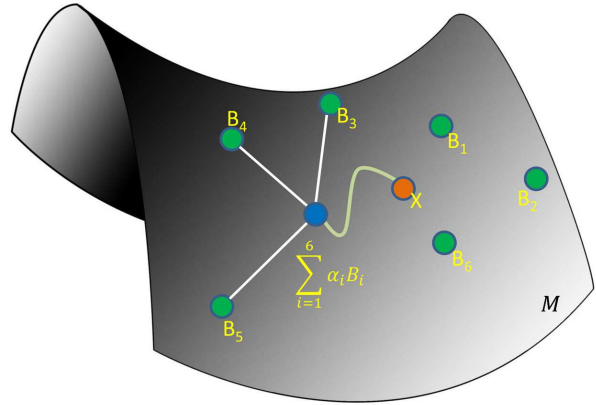}	  
    \caption{\small An illustration of our Riemannian dictionary learning and sparse coding. For the SPD manifold $\manifold$ and learned SPD basis matrices $\atom_i$ on the manifold, our sparse coding objective seeks a non-negative sparse linear combination $\sum_i\alpha_i\atom_i$ of the $\atom_i$'s that is closest (in a geodesic sense) to the given input SPD matrix $X$. }	  
    \label{fig:rsp}		
\end{figure}

We focus on SPD matrices for dictionary learning with sparse coding (DLSC) -- a powerful data representation tool in computer vision and machine learning~\cite{mairalbook} that enables state-of-the-art results for a variety of applications~\cite{guha2012learning,wright2009robust}. 

Given a training set, traditional (Euclidean) dictionary learning problem seeks an overcomplete set of basis vectors (the dictionary) that can be linearly combined to sparsely represent each input data point; finding a sparse representation given the dictionary is termed~\emph{sparse coding}. While sparse coding was originally conceived for Euclidean vectors, there have been recent extensions of the setup to other data geometries, such as  histograms~\cite{cherian2014nearest}, Grassmannians~\cite{harandi2013dictionary}, and SPD matrices~\cite{harandi2012sparse,lilog,sivalingam2010tensor,sra2011generalized}. The focus of this paper is on dictionary learning and sparse coding of SPD matrix data using a novel mathematical model inspired by the geometry of SPD matrices.

SPD matrices are an open subset of the Euclidean space of symmetric matrices. This may lead one to believe that essentially treating them as vectors may suffice. However, there are several specific properties that applications using SPD matrices demand. For example, in DT-MRI the semidefinite matrices are required to be at infinite distances from SPD matrices~\cite{arsigny2006log}. Using this and other geometrically inspired motivation a variety of non-Euclidean distance functions have been used for SPD matrices---see e.g.,~\cite{arsigny2006log,sra2011positive,moakher2006symmetric}. The most widely used amongst these is the affine invariant Riemannian metric~\cite{moakher2006symmetric,pennec2006riemannian}, the only intrinsic Riemannian distance that corresponds to a geodesic distance on the manifold of SPD matrices. 

In this paper, we study dictionary learning and sparse coding of SPD matrices in their natural Riemannian geometry. Compared to the Euclidean setup, their Riemannian geometry, however, poses some unique challenges: (i) the manifold defined by this metric is (negatively) curved~\cite{rothaus1960domains}, and thus the geodesics are no more straight-lines; and (ii) in contrast to the Euclidean DLSC formulations, the objective function motivated by the Riemannian distance is not convex in either its sparse coding part or in the dictionary learning part. We present some theoretical properties of our new DLSC formulation and mention a situation of purely theoretical interest where the formulation is convex. Figure~\ref{fig:rsp} conceptually characterizes our goal. 
%

The key contributions of this paper are as follows.
\begin{itemize}
\item \emph{Formulation:} We propose a new model to learn a dictionary of SPD atoms; each data point is represented as a nonnegative sparse linear combination of SPD atoms from this dictionary. The quality of the resulting representation is measured by the squared intrinsic Riemannian distance. 
\item \emph{Optimization:} The main challenge in using our formulation is its higher computational cost relative to Euclidean sparse coding. However, we describe a simple and effective approach for optimizing our objective function. Specifically, we propose a dictionary learning algorithm on SPD atoms via conjugate gradient descent on the product manifold generated by the SPD atoms in the dictionary.
\end{itemize}
A forerunner to this paper appeared in~\cite{cherian2014riemannian}. The current paper differs from our conference paper in the following major aspects: (i) we propose a novel dictionary learning formulation and an efficient solver for it; and (ii) extensive new experiments using our dictionary learning are also included and the entire experimental section re-evaluated under our new setup while also including other datasets and evaluation metrics.

To set the stage for presenting our contributions, we first review key tools and metrics from Riemannian geometry that we use to develop our ideas. Then we survey some recent methods suggested for sparse coding using alternative similarity metrics on SPD matrices. Throughout we work with real matrices; extension to Hermitian positive definite matrices is straightforward. The space of $d\times d$ SPD matrices is denoted as $\spd{d}$, symmetric matrices by $\sym{d}$, and the space of (real) invertible matrices by $\genlinear(d)$. By $\Log(X)$, for $X\in\spd{}$, we mean the principal matrix logarithm and $\logdet{X}$ denotes the scalar logarithm of the matrix determinant.

\section{Preliminaries}
\label{sec:background}
We provide below a brief overview of the Riemannian geometry of SPD matrices. A manifold $\manifold$ is a set of points endowed with a locally-Euclidean structure. A~\emph{tangent vector} at $P\in\manifold$ is defined as the tangent to some curve in the manifold passing through $P$. A tangent space $\tangent{P}\manifold$ defines the union of all such tangent vectors to all possible such curves passing through $P$; the point $P$ is termed the~\emph{foot} of this tangent space. The dimensionality of $\tangsp$ is the same as that of the manifold. It can be shown that the tangent space is isomorphic to the Euclidean space~\cite{absil2009optimization}; thus, they provide a locally-linear approximation to the manifold at its foot. 

A manifold becomes a Riemannian manifold if its tangent spaces are endowed with a smoothly varying inner product. The Euclidean space, endowed with the classical inner product defined by the trace function (i.e., for two points $X,Y\in\semipd{d}, \langle X, Y\rangle = \trace(XY)$), is a Riemannian manifold. Recall that an SPD matrix has the property that all its eigenvalues are real and positive, and it belongs to the interior of a convex self-dual cone. Since for $d\times d$ SPD matrices, this cone is a subset of the $\half d(d+1)$-dimensional Euclidean space of symmetric matrices, the set of SPD matrices naturally forms a Riemannian manifold under the trace metric. However, under this metric, the SPD manifold is not~\emph{complete}~\footnote{A space is a complete metric space if all~\emph{Cauchy} sequences are convergent within that space.}. This is because, the trace metric does not enclose all Cauchy sequences originating from the interior of the SPD cone~\cite{cheng2013novel}.

A possible remedy is to change the geometry of the manifold such that positive semi-definite matrices (which form the closure of SPD matrices for Cauchy sequences) are at an infinite distance to points in the interior of the SPD cone. This can be achieved by resorting to the the classical~\emph{log-barrier function} $g(P) = -\log\det(P)$, popular in the optimization community in the context of interior point methods~\cite{nesterov2002riemannian}. \textcolor{blue}{The trace metric can be modified to the new geometry induced by the barrier function by incorporating the curvature through its Hessian operator at point $P$ in the direction $Z$ given by $H_P(Z) = g''(P)(Z) = P^{-1} Z P^{-1}$.} The Riemannian metric at $P$ for two points $Z_1, Z_2\in\tangsp$ is thus defined as:
\begin{equation}
\langle Z_1, Z_2\rangle_P = \langle H_PZ_1, Z_2 \rangle = \trace(P^{-1}Z_1P^{-1}Z_2).
\label{eq:riem_metric}
\end{equation}

There are two fundamental operations that one needs for computations on Riemannian manifolds:  (i) the exponential map $\Exp_P:\sym{d}\to\spd{d}$; and (ii) the logarithmic map $\Log_P=\Exp_P^{-1}:\spd{d}\to\sym{d}$, where $P\in\spd{d}$. The former projects a symmetric point on the tangent space onto the manifold (termed a~\emph{retraction}), the latter does the reverse. Note that these maps depend on the point $P$ at which the tangent spaces are computed. In our analysis, we will be measuring distances assuming $P$ to be the identity matrix, $\eye$, in which case we will omit the subscript.

Note that the Riemannian metric provides a measure for computing distances on the manifold. Given two points on the manifold, there are infinitely many paths connecting them, of which the shortest path is termed the~\emph{geodesic}. It can be shown that the SPD manifold under the Riemannian metric in~\eqref{eq:riem_metric} is non-positively curved (Hadamard manifold) and has a unique geodesic between every distinct pair of points~\cite{lang2012fundamentals}[Chap. 12],~\cite{bhatia07}[Chap. 6]. For $X,Y\in\spd{d}$, there exists a closed form for this~\emph{geodesic distance}, given by
\begin{equation}
  \label{eq:riem}
\riem(X,Y) = \fnorm{\Log(X^{-1/2} Y X^{-1/2})},
\end{equation}
where $\Log$ is the matrix logarithm. It can be shown that this distance is invariant to affine transformations of the input matrices~\cite{pennec2006riemannian} and thus is commonly referred to as the~\emph{Affine invariant Riemannian metric}. In the sequel, we will use~\eqref{eq:riem} to measure the distance between input SPD matrices and their sparse coded representations obtained by combining dictionary atoms.

\section{Related Work}
\label{sec:relatedwork}
Dictionary learning and sparse coding of SPD matrices has received significant attention in the vision community due to the performance gains it brings to the respective applications~\cite{mairalbook,harandi2012sparse,sra2011generalized}. Given a training dataset $\mathcal{X}$, the DLSC problem seeks a dictionary $\dict$ of basis atoms, such that each data point $x\in\mathcal{X}$ can be approximated by a sparse linear combination of these atoms while minimizing a suitable loss function. Formally, the DLSC problem can be abstractly written as
\begin{equation}
\min_{\dict, \theta_x \forall x\in\mathcal{X}}\sum_{\forall x\in\mathcal{X}} \loss(x, \dict, \theta_x) + \lambda\sparse(\theta_x),
\label{eq:abstract_eqn}
\end{equation}
where the loss function $\loss$ measures the approximation quality obtained by using the ``code'' $\theta_x$, while $\lambda$ regulates the impact of the sparsity penalty $\sparse(\theta_x)$. 

As alluded to earlier, the manifold geometry hinders a straightforward extension of classical DLSC techniques (such as~\cite{aharon2006img,denoising2006}) to data points drawn from a manifold. Prior methods typically use surrogate similarity distances that bypass the need to operate within the intrinsic Riemannian geometry, e.g., (i) by adapting information geometric divergence measures such as the log-determinant divergence or the Stein divergence, (ii) by using extrinsic metrics such as the log-Euclidean metric, and (iii) by relying on the kernel trick to embed the SPD matrices into a suitable RKHS. We briefly review each of these schemes below.


\noindent{\bf{Statistical measures:}} In~\cite{sivalingam2010tensor} and~\cite{sivalingam2011positive}, a DLSC framework is proposed based on the log-determinant divergence (Burg loss) to model the loss. \textcolor{blue}{For two SPD matrices $X,Y\in\spd{d}$, this divergence has the following form $\burg(X,Y)=\trace(XY^{-1}) -\logdet{XY^{-1}} -d$.} Since this divergence acts as a base measure for the Wishart distribution~\cite{cherian2011dirichlet}---a natural probability density on SPD matrices---a loss defined using it is statistically well-motivated. The sparse coding formulation using this loss reduces to a \textsc{MaxDet} optimization problem~\cite{sivalingam2010tensor} and is solved using interior-point methods. Unsurprisingly, the method is often seen to be computationally demanding even for moderately sized covariances (more than $5\times 5$). Ignoring the specific manifold geometry of SPD matrices, one may directly extend the Euclidean DLSC schemes to the SPD setting. However, a na\"{i}ve use of Euclidean distance on SPD matrices is usually found inferior in performance. It is argued in~\cite{sra2011generalized} that approximating an SPD matrix as sparse conic combinations of positive semi-definite rank-one outer-products of the Euclidean dictionary matrix atoms leads to improved performance under the Euclidean loss. However, the resulting dictionary learning subproblem is nonconvex and the reconstruction quality is still measured using a Euclidean loss. Further, discarding the manifold geometry is often seen to showcase inferior results compared to competitive methods~\cite{cherian2014riemannian}.


\noindent{\bf{Differential geometric schemes:}} Among the computationally efficient variants of Riemannian metrics, one of the most popular is the log-Euclidean metric $\lerm$~\cite{arsigny2006log} defined for $X,Y\in\spd{d}$ as $\lerm(X,Y) := \fnorm{\Log(X) - \Log(Y)}$. The $\Log$ operator maps an SPD matrix isomorphically and diffeomorphically into the flat space of symmetric matrices; the distances in this space are Euclidean. DLSC with the squared log-Euclidean metric has been proposed in the past~\cite{guo2010action} with promising results. A similar framework was suggested recently~\cite{ho2013nonlinear} in which a local coordinate system is defined on the tangent space at the given data matrix. While, their formulation uses additional constraints that make their framework coordinate independent, their scheme restricts sparse coding to specific problem settings, such as an affine coordinate system.

\noindent{\bf{Kernelized Schemes:}} In~\cite{harandi2012sparse}, a kernelized DLSC framework is presented for SPD matrices using the Stein divergence~\cite{sra2011positive} for generating the underlying kernel function. \textcolor{blue}{For two SPD matrices $X,Y$, the Stein divergence is defined as $\stein(X,Y) = \logdet{\half (X+Y)}-\half\logdet{XY}$}. As this divergence is a statistically well-motivated similarity distance with strong connections to the natural Riemannian metric (~\cite{cichocki2015log,sra2011positive}) while being computationally superior, performances using this measure are expected to be similar to those using the Riemannian metric~\cite{cherian2013jensen}. However, this measure does not produce geodesically exponential kernels for all bandwidths~\cite{sra2011positive} making it less appealing theoretically. In~\cite{cvpr_hartley,lilog} kernels based on the log-Euclidean metric are proposed.  A general DLSC setup is introduced for the more general class of Riemannian manifolds in~\cite{harandi2009riemannian}. The main goal of all these approaches is to linearize the curved manifold geometry by projecting the SPD matrices into an infinite dimensional Hilbert space as defined by the respective kernel. However, as recently shown theoretically in~\cite{cvpr_hartley,feragen2014geodesic} most of the curved Riemannian geometries (including the the span of SPD matrices) do not have such kernel maps, unless the geometry is already isometric to the Euclidean space (as in the case of the log-Euclidean metric). This result severely restricts the applicability of traditional kernel methods to popular Riemannian geometries (which are usually curved), thereby providing strong motivation to study the standard machine learning algorithms within their intrinsic geometry --- as is done in the current paper.


In light of the above summary, our scheme directly uses the affine invariant Riemannian metric to design our sparse reconstruction loss. To circumvent the computational difficulty we propose an efficient algorithm based on spectral projected gradients for sparse coding, while we use an adaptation of the non-linear conjugate gradient on manifolds for dictionary learning. Our experiments demonstrate that our scheme is computationally efficient and provides state of the art results on several computer vision problems that use covariance matrices.


\section{Problem Formulation}
\label{sec:dlsc}
Let $\dataset=\set{\datapt_1,\datapt_2,\cdots, \datapt_N}$ denote a set of $N$ SPD data matrices, where each $\datapt_i\in\spd{d}$. Let $\manifold^d_n$ denote the product manifold obtained from the Cartesian product of $n$ SPD manifolds, i.e., $\manifold^d_n=\spd{d}\times^n\subset\reals{d\times d\times n}$. Our goals are (i) to learn a third-order tensor (dictionary) $\dict\in\manifold^d_n$ in which each slice represents an SPD dictionary atom $\atom_j\in\spd{d},j=1,2,\cdots, n$; and (ii) to approximate each $\datapt_i$ as a sparse conic combination of atoms in $\dict$; i.e., $\datapt_i \sim \dict\alpha_i$ where $\alpha_i\in\reals{n}_{+}$ and $\dict v := \sum_{i=1}^n v_i\atom_i$ for an $n$-dimensional vector $v$. With this notation, our joint DLSC objective is
\begin{align}
	\min_{\substack{\dict\in\manifold^d_n,\\\alpha\in\reals{n\times N}_+}} \half\sum_{j=1}^N\riem^2\left(\datapt_j,\dict\alpha_j\right) + \sparse(\alpha_j) + \Omega(\dict),
	\label{eq:problem_p1}
\end{align}
where $\sparse$ and $\Omega$ are regularizers on the coefficient vectors $\alpha_j$ and the dictionary tensor respectively. 

Although formulation~\eqref{eq:problem_p1} may look complicated, it is a direct analogue of the vectorial DLSC setup to matrix data. For example, instead of learning a dictionary matrix in the vectorial DLSC, we learn a third-order tensor dictionary since our data $X$ are now matrices. The need to constrain the sparse coefficients to the non-negative orthant is required to make sure the linear combination of SPD atoms stays within the SPD cone. 
However, in contrast to the vectorial DLSC formulations for which the subproblems on the dictionary learning and sparse coding are convex separately, the problem in~\eqref{eq:problem_p1} is neither convex in itself, nor are its subproblems convex.

From a practical point of view, this lack of convexity it is not a significant concern as all we need is a set of dictionary atoms which can sparse code the input. To this end, we propose below an alternating minimization (descent) scheme that alternates between locally solving the dictionary learning and sparse coding sub-problems, while keeping fixed the variables associated with the other. \textcolor{blue}{A full theoretical analysis of the convergence of this nonconvex problem is currently beyond the scope of this paper and of most versions of nonconvex analysis known to us. However, what makes the method interesting and worthy of future analysis is that empirically it converges quite rapidly as shown in Figure 5.}

%


\subsection{Dictionary Learning Subproblem}
\label{sec:dictionary_learning_subproblem}
Assuming the coefficient vectors $\alpha$ available for all the data matrices, the subproblem for updating the dictionary atoms can be separated from~\eqref{eq:problem_p1} and written as:
\begin{align}	
		&\min_{\dict\in\manifold^d_n} \Theta(\dict):=\half\sum_{j=1}^N\riem^2\left(\datapt_j, \dict\alpha_j\right) + \Omega(\dict),\\
			&=\half\sum_{j=1}^N\fnorm{\Log\left(\datapt_j^{-\half}\left(\dict\alpha_j\right)\datapt_j^{-\half}\right)}^2 + \Omega(\dict).
	\label{eq:dict_subproblem_1}
\end{align}


\subsubsection{Regularizers} Before delving into algorithms for optimizing~\eqref{eq:dict_subproblem_1}, let us recall a few potential regularizers $\Omega$ on the dictionary atoms, which are essential to avoid overfitting the dictionary to the data. For SPD matrices, we have several regularizers available, such as: (i) the largest eigenvalue regularizer $\Omega(\dict)=\sum_{i}\enorm{\atom_i}^2$, (ii) deviation of the dictionary from the identity matrix $\Omega(\dict) = \sum_i \fnorm{\atom_i - \eye}^2$, (iii) the Riemannian elasticity regularizer~\cite{pennec2005riemannian} which measures the Riemannian deformation of the dictionary from the identity matrix $\Omega(\dict) = \sum_{i}\fnorm{\log(\atom_i)-\log(\eye)}^2=\riem(\atom_i, \eye)^2$, and (iv) the trace regularizer, i.e., $\Omega(\dict) = \lambda_{\dict}\sum_{i=1}^n\trace(\atom_i)$, for a regularization parameter $\lambda_{\dict}$. In the sequel, we use the unit-trace regularizer as it is simpler and performs well empirically.


\subsubsection{Optimizing Dictionary Atoms} 

Among several first-order alternatives for optimizing over the SPD atoms (such as the steepest-descent, trust-region methods~\cite{absil2007trust}, etc.), the Riemannian Conjugate Gradient (CG) method~\cite{absil2009optimization}[Chap.8], was found to be empirically more stable and faster. Below, we provide a short exposition of the CG method in the context of minimizing over $\dict$ which belongs to an SPD product manifold.

For an arbitrary non-linear function $\theta(x),\ x\in\reals{n}$, the CG method uses the following recurrence at step $k+1$
\begin{equation}
x_{k+1} = x_k + \gamma_k\xi_k,
\label{eq:cg-base}
\end{equation}
where the direction of descent $\xi_k$ is
\begin{equation}
\xi_k = -\grad \theta(x_k) + \mu_k\xi_{k-1},
\label{eq:cg-dirs}
\end{equation}
with $\grad \theta(x_k)$ defining gradient of $\theta$ at $x_k$ ($\xi_0=-\grad \theta(x_0)$), and $\mu_k$ given by
\begin{equation}
\mu_k = \frac{(\grad \theta(x_k))^T(\grad \theta(x_k) - \grad \theta(x_{k-1}))}{\grad \theta(x_{k-1})^T \grad \theta(x_{k-1})},
\label{eq:next-step}
\end{equation}
The step-size $\gamma_k$ in~\eqref{eq:cg-base} is usually found via an efficient line-search method~\cite{bertsekas99}. It can be shown that~\cite{bertsekas99}[Sec.1.6] when $\theta$ is quadratic with a Hessian $Q$, the directions generated by~\eqref{eq:cg-dirs} will be Q-conjugate to previous directions of descent $\xi^0, \xi^1,\cdots,\xi^{k-1}$; thereby~\eqref{eq:cg-base} providing the exact minimizer of $f$ in fewer than $d$ iterations ($d$ is the manifold dimension).

For $\dict\in\manifold^d_n$ and referring back to~\eqref{eq:dict_subproblem_1}, the recurrence in~\eqref{eq:cg-base} will use the Riemannian retraction~\cite{absil2009optimization}[Chap.4] and the gradient $\grad \Theta(\dict_k)$ will assume the Riemannian gradient (here we use $\dict_k$ to represent the dictionary tensor at the $k$-th iteration). This leads to an important issue: the gradients $\grad \Theta(\dict_k)$  and $\grad \Theta(\dict_{k-1})$ belong to two different tangent spaces $\tangent{\dict_k}\manifold$ and $\tangent{\dict_{k-1}}\manifold$ respectively, and thus cannot be combined as in~\eqref{eq:next-step}. Thus, following~\cite{absil2009optimization}[Chapter 8] we resort to vector transport -- a scheme to transport a tangent vector at $P\in\manifold$ to a point $\Exp_P(S)$ where $S\in\tangsp$ and $\Exp$ is the exponential map. The resulting formula for the direction update becomes
\begin{equation}
\xi_{\dict_{k}} = -\grad \Theta(\dict_{k}) + \mu_{k}\vectrans{\gamma_k\xi_{k-1}}{\xi_{k-1}},
\end{equation}
where
{\small
\begin{equation}
\mu_{k} = \frac{\left\langle\grad \Theta(\dict_{k}), \grad \Theta(\dict_{k}) - \vectrans{\gamma_k\xi^{k-1}}{\grad \Theta(\dict_{k-1})} \right\rangle}{\left\langle \grad \Theta(\dict_{k-1}),  \grad \Theta(\dict_{k-1})\right\rangle}.
\label{eq:riem-mu-update}
\end{equation}
}
Here for $Z_1, Z_2\in\tangsp$, the map $\vectrans{Z_1}{Z_2}$ defines the vector transport given by:
\begin{equation}
\vectrans{Z_1}{Z_2} = \frac{d}{dt}\exp_P(Z_1 + tZ_2)\bigg|_{t=0}.
\end{equation}


The remaining technical detail is the expression for the Riemannian gradient $\grad \Theta(\dict)$, which we derive next.
\subsubsection{Riemannian Gradient}
The following lemma connects the Riemannian gradient to the Euclidean gradient of $\Theta(\dict)$ in~\eqref{eq:dict_subproblem_1}.
\begin{lemma}
For a dictionary tensor $\dict\in\manifold^d_n$, let $\Theta(\dict)$ be a differentiable function. Then the Riemannian gradient $\grad \Theta(\dict)$ satisfies the following condition:
\begin{equation}
\langle \grad \Theta(\dict), \zeta \rangle_{\dict} = \langle \nabla\Theta(\dict), \zeta\rangle_{\eye}, \forall \zeta\in\tangsp^d_n,
\label{eq:grad_euc_grad}
\end{equation}
where $\nabla \Theta(\dict)$ is the Euclidean gradient of $\Theta(\dict)$. The Riemannian gradient for the $i$-th dictionary atom is given by $\grad_i\Theta(\dict) = \atom_i\nabla_{\atom_i}\Theta(\dict) \atom_i$.
\end{lemma}
\begin{proof}
See \cite{absil2009optimization}[Chap. 5]. The latter expression is obtained by substituting the inner product on the LHS of~\eqref{eq:grad_euc_grad} by its definition in~\eqref{eq:riem_metric}.
\end{proof}

We can derive the Euclidean gradient $\nabla \Theta(\dict)$ as follows: let $\dataZ_j=\datapt_j^{-\half}$ and $M_j(\dict)=\dict\alpha_j=\sum_{i=1}^n\alpha_j^i\atom_i$. Then,
\begin{equation}
	\Theta(\dict) = \half\sum_{j=1}^N \trace(\Log(\dataZ_j M_j(\dict)\dataZ_j)^2)+ \lambda_{\dict}\sum_{i=1}^n\trace(\atom_i).					
	\label{eq:dict_subproblem_2}
\end{equation}
The derivative $\nabla_{\atom_i}\Theta(\dict)$ of~\eqref{eq:dict_subproblem_2} w.r.t. to atom $\atom_i$ is:
\begin{align}
	\sum_{j=1}^N\alpha^i_j\bigl(\dataZ_j \Log(M_j(\dict)) \bigl(M_j(\dict)\bigr)^{-1} \dataZ_j\bigr) + \lambda_{\dict}I.
\label{eq:euc_gradient_dict_problem}
\end{align}

%
%

\subsection{Sparse Coding Subproblem}
\label{sec:sparse_coding_subproblem}
Referring back to~\eqref{eq:problem_p1}, let us now consider the sparse coding subproblem. Suppose we have a dictionary tensor $\dict$ available. For a data matrix $\datapt_j\in\spd{d}$ our sparse coding objective is to solve
\begin{equation}
  \label{eq:sc_subproblem_1}
  \begin{split}
	&\min_{\alpha_j \ge 0}\quad\phi(\alpha_j) := \half\riem^2\left(\datapt_i,\dict\alpha_j\right) + \sparse(\alpha_j)\\
    &=\half\fnorm{\Log{\nlsum_{i=1}^n \alpha^i_j \datapt^{-\half}\atom_j \datapt^{-\half}}}^2 + \sparse(\alpha_j),
  \end{split}
\end{equation}
where $\alpha^i_j$ is the $i$-th dimension of $\alpha_j$ and $\sparse$ is a sparsity inducing function. For simplicity, we use the sparsity penalty $\sparse(\alpha) = \lambda\|\alpha\|_1$, where $\lambda > 0$ is a regularization parameter. Since we are working with $\alpha \ge 0$, we replace this penalty by $\lambda\sum_i \alpha_i$, which is differentiable.
\medskip

The subproblem~\eqref{eq:sc_subproblem_1} measures reconstruction quality offered by a sparse non-negative linear combination of the atoms to a given input point $X$. It will turn out (see experiments in Section~\ref{sec:expts}) that the reconstructions obtained via this model actually lead to significant improvements in performance over sparse coding models that ignore the rich geometry of SPD matrices. But this gain comes at a price: model~\eqref{eq:sc_subproblem_1} is a nontrivial to optimize; it remains difficult even if we take into account geodesic convexity of $\riem$.

While in practice this nonconvexity does not seem to hurt our model, we show below a surprising but intuitive constraint under which Problem~\eqref{eq:sc_subproblem_1} actually becomes convex. The following lemma will be useful later.


\begin{lemma}
  \label{lem:grad}
  Let $B$, $C$, and $X$ be fixed SPD matrices. Consider the function $f(x)=\riem^2(xB+C,X)$. The derivative $f'(x)$ is given by    
  \begin{equation}
    \label{eq:4}
    f'(x)=2\trace(\log(S(xB+C)S)S^{-1}(xB+C)^{-1}BS),    
  \end{equation}  
  where $S\equiv X^{-\half}$.
\end{lemma}
\begin{proof}
  Introduce the shorthand $M(x)\equiv xB+C$, from definition~\eqref{eq:riem} and using $\fnorm{Z}^2 = \trace(Z^TZ)$ we have
  \begin{equation*}
    f(x) = \trace( \log(SM(x)S)^T\log(SM(x)S)).
  \end{equation*}
  The chain-rule of calculus then immediately yields
  \begin{equation*}
    f'(x) = 2\trace(\log(SM(x)S)(SM(x)S)^{-1}SM'(x)S),
  \end{equation*}
  which is nothing but~\eqref{eq:4}.
\end{proof}

As a brief digression, let us mention below an interesting property of the sparse-coding problem. We do not exploit this property in our experiments, but highlight it here for its theoretical appeal.
\begin{theorem}
  \label{thm:opt}
  The function $\phi(\alpha) := \riem^2(\sum_i\alpha_iB_i,X)$ is convex on the set
  \begin{equation}
    \label{eq:10}
    \mathcal{A} := \lbrace \alpha \mid \nlsum_i \alpha_iB_i \preceq X,\ \text{and } \alpha \ge 0\rbrace.
  \end{equation}
\end{theorem}
\begin{proof}
See Appendix~\ref{app:proof_of_theorem_opt}.
\end{proof}

\relax

Let us intuitively describe what Theorem~\ref{thm:opt} is saying. While sparsely encoding data we are trying to find sparse coefficients $\alpha_1,\ldots,\alpha_n$, such that in the ideal case we have $\sum_i \alpha_i B_i = X$. But in general this equality cannot be satisfied, and one only has $\sum_i\alpha_iB_i\approx X$, and the quality of this approximation is measured using $\phi(\alpha)$ or some other desirable loss-function. The loss $\phi(\alpha)$ from~\eqref{eq:sc_subproblem_1} is nonconvex while convexity is a ``unilateral'' property---it lives in the world of inequalities rather than equalities~\cite{urruty}. And it is known that SPD matrices in addition to forming a manifold also enjoy a rich conic geometry that is endowed with the L\"owner partial order. Thus, instead of seeking arbitrary approximations $\sum_i\alpha_iB_i \approx X$, if we limit our attention to those that underestimate $X$ as in~\eqref{eq:10}, we might benefit from the conic partial order. It is this intuition that Theorem~\ref{thm:opt} makes precise.


\subsubsection{Optimizing Sparse Codes}
Writing $M(\alpha_p) = \alpha_pB_p+\nlsum_{i\neq p}\alpha_iB_i$ and using Lemma~\ref{lem:grad} we obtain 
\begin{equation}
  \label{eq:3}
  \frac{\partial \phi(\alpha)}{\partial \alpha_p} = \trace\left(\log\bigl(SM(\alpha_p)S\bigr)\bigl(SM(\alpha_p)S\bigr)^{-1}SB_pS\right) + \lambda.
\end{equation}
Computing~\eqref{eq:3} for all $\alpha$ is the dominant cost in a gradient-based method for solving~\eqref{eq:problem_p1}. We present pseudocode (Alg.~\ref{alg:grad}) that efficiently implements the gradient for the first part of~\eqref{eq:3}. The total cost of Alg.~\ref{alg:grad} is $O(nd^2)+O(d^3)$---a na\"ive implementation of~\eqref{eq:3} costs $O(nd^3)$, which is substantially more expensive.

\begin{algorithm}  
  \SetAlgoLined
  \KwIn{$B_1,\ldots,B_n, X \in \spd{d}$, $\alpha \ge 0$}
  $S \gets X^{-1/2}$; $M \gets \nlsum_{i=1}^n\alpha_i B_i$\;
  $T \gets S\log(SMS)(MS)^{-1}$\;
  \For{$i=1$ \KwTo $n$}{
    $g_i \gets \trace(TB_p)$\;
  }
  \KwRet{$g$}
  \caption{Efficient computation of gradients}
  \label{alg:grad}
\end{algorithm}

 Alg.~\ref{alg:grad} in conjunction with a gradient-projection scheme essentially  runs the iteration
\begin{equation}
  \label{eq:5}
  \alpha^{k+1} \gets \proj{\alpha^k - \eta_k\nabla\phi(\alpha^k)},\qquad k=0,1,\ldots,
\end{equation}
where $\proj{\cdot}$ denotes the projection operator defined as
\begin{equation}
  \label{eq:11}
  \proj{\alpha} \equiv \alpha \mapsto \argmin\nolimits_{\alpha'}\tfrac{1}{2}\|\alpha'-\alpha\|_2^2,\quad\text{s.t.}\ \alpha' \in \mathcal{A}.
\end{equation}
Iteration~\eqref{eq:5} has three major computational costs: (i) the stepsize $\eta_k$; (ii)  the gradient $\nabla\phi(\alpha^k)$; and (iii) the projection~\eqref{eq:11}. Alg.~\ref{alg:grad} shows how to efficiently obtain the gradient. The projection task~\eqref{eq:11} is a special least-squares (dual) semidefinite program (SDP), which can be solved using any SDP solver or by designing a specialized routine. To avoid the heavy computational burden imposed by an SDP, we drop the constraint $\alpha\in \mathcal{A}$; this sacrifices convexity but makes the computation vastly easier, since with this change, we simply have $\proj{\alpha} = \max(0,\alpha)$.

In~\eqref{eq:5}, it only remains to specify how to obtain the stepsize $\eta_k$. There are several choices available in the nonlinear programming literature~\cite{bertsekas99} for choosing $\eta_k$, but most of them can be quite expensive. In our quest for an efficient sparse coding algorithm, we choose to avoid expensive line-search algorithms for selecting $\eta_k$ and prefer to use the Barzilai-Borwein stepsizes~\cite{bb88}, which can be computed in closed form and lead to remarkable gains in performance~\cite{bb88,schmidt09}. In particular, we use the Spectral Projected Gradient (SPG) method~\cite{birgin01} by adapting a simplified implementation of~\cite{schmidt09}.

SPG runs iteration~\eqref{eq:5} using Barzilai-Borwein stepsizes with an occasional call to a nonmontone line-search strategy to ensure convergence of $\{\alpha^k\}$. Without the constraint $\alpha' \in \mathcal{A}$, we cannot guarantee anything more than a stationary point of~\eqref{eq:problem_p1}, while if we were to use the additional constraint then we can even obtain global optimality for iterates generated by~\eqref{eq:5}.


\section{Experiments}
\label{sec:expts}
In this section, we provide experiments on simulated and real-world data demonstrating the effectiveness of our algorithm compared to the state-of-the-art DLSC methods on SPD valued data. First, we demonstrate results on simulated data analyzing the performance of our framework for various settings. This will precede experiments on standard benchmark datasets.

\subsection{Comparison Methods}
In the experiments to follow, we will denote dictionary learning and sparse coding algorithms by DL and SC respectively. 
We will compare our Riemannian (Riem) formulation against combinations of several state-of-the-art DLSC methods on SPD matrices, namely using (i) log-Euclidean (LE) metric for DLSC~\cite{guo2010action}, (ii) Frobenius norm (Frob) which discards the manifold structure, (iii) kernel
methods such as the Stein-Kernel [18] proposed in [12] and the log-Euclidean kernel [13]. 


\subsection{Simulated Experiments} 
In this subsection, we evaluate in a controlled setting, some of the properties of our Riemannian sparse coding scheme. For all our simulations, we used covariances generated from data vectors sampled from a zero-mean unit covariance normal distribution. For each covariance sample, the number of data vectors is chosen to be ten times its dimensionality. For fairness of the comparisons, we adjusted the regularization parameters of the sparse coding algorithms so that the codes generated are approximately 10\% sparse. The plots to follow show the performance averaged over 50 trials. \textcolor{blue}{Further, all the algorithms in this experiment used the SPG method to solve their respective formulations so that their performances are comparable. The intention of these timing comparisons is to empirically point out the relative computational complexity of our Riemannian scheme against the baselines rather than to show exact computational times. For example, for the comparisons against the method Frob-SC, one can vectorize the matrices and then use a vectorial sparse coding scheme. In that case, Frob-SC will be substantially faster, and incomparable to our scheme as it solves a different problem.} In these experiments, we will be using the classification accuracy as the performance metric. Our implementations are in MATLAB and the timing comparisons used a single core Intel 3.6GHz CPU.

\subsubsection{Increasing Data Dimensionality} 
While DT-MRI applications typically use small SPD matrices ($3\times 3$), the dimensionality is very diverse for other applications in computer vision. For example, Gabor covariances for face recognition uses about $40$-dimensional SPD matrices~\cite{pang2008gabor}, while even larger covariance descriptors are becoming common~\cite{harandi2014manifold}. The goal of this experiment is to analyze the scalability of our sparse coding setup against an increasing size of the data matrices. To this end, we fixed the number of dictionary atoms to be 200, while increased the matrix dimensionality from 3 to 100. Figure~\ref{fig:time_against_dim} plots the time-taken by our method against the na\"{i}ve Frob-SC method (although it uses the SPG method for solution). The plot shows that the extra computations required by our Riem-SC is not substantial compared to Frob-SC. 


\subsubsection{Increasing Dictionary Size} 
In this experiment, we compare the scalability of our method to work with larger dictionary tensors. To this end, we fixed the data dimensionality to 10, while increased the number of dictionary atoms from 20 to 1000. Figure~\ref{fig:time_against_atoms} plots the time-taken against the dictionary size. As is expected, the sparse coding performance for all the kernelized schemes drops significantly for larger dictionary sizes, while our scheme performs fairly. 

\subsubsection{Increasing Sparsity Regularization} 
In this experiment, we decided to evaluate the effect of the sparsity promoting regularization $\lambda$ in~\eqref{eq:sc_subproblem_1}. To this end, we generated a dictionary of 100 atoms from covariances of Gaussian random variables. Later, 1000 SPD matrices are produced using conic combinations of randomly selected atoms. We used an active size of 10 dictionary atoms for all the SPD matrices. After adding random SPD noise to each matrix, we used half of them for learning the dictionary, while the other half was used for evaluating the sparsity regularization. We increased $\lambda$ from $10^{-5}$ to $10^{5}$ at steps of 10. In Figure~\ref{fig:sparsity_against_lambda}, we plot the sparsity (i.e., number of non-zero coefficients/size of coefficients) for varying $\lambda$. We see that while the lower values of $\lambda$ does not have much influence on sparsity, as $\lambda$ increases beyond a certain threshold, sparsity increases. A similar trend is seen for increasing data dimensionality. However, we find that the influence of $\lambda$ starts diminishing as the dimensionality increases. For example, sparsity plateaus after $3\%$ for 5-dimensional data, while this happens at nearly $15\%$ for 20-dimensional data. The plateauing of sparsity is not unexpected and is directly related to the Riemannian metric that we use -- our loss will prevent all the sparse coefficients from going to zero simultaneously as in such a case the objective will tend to infinity. Further, as the matrix dimensionality increases, it is more likely that the data matrices become ill-conditioned. As a result, this plateau-ing happens much earlier than for better conditioned matrices (as in the case of 5-dimensional matrices in Figure~\ref{fig:sparsity_against_lambda}). 

In Figure~\ref{fig:sparsity_against_lambda_compare_LE}, we contrast the sparsity pattern produced by our Riemannian sparse coding (Riem-DL + Riem-SC) scheme against that of the traditional sparse coding objective using log-Euclidean sparse coding (LE-DL + LE-SC), for 20-dimensional SPD data. As is expected, the log-Euclidean DL follows the conventional convergence patterns in which sparsity goes to zero for larger values of the regularization. Since for larger regularizations, most of the coefficients in our Riem-SC have low values, we can easily discard them by thresholding. However, we believe this difference in the sparsity patterns needs to be accounted for when choosing the regularization parameters for promoting sparsity in our setup.

\subsubsection{Convergence for Increasing Dimensionality}
In this experiment, we evaluate the convergence properties of our dictionary learning sub-problem based on the Riemannian conjugate gradient scheme. To this end, we used the same setup as in the last experiment using data generated by a pre-defined dictionary, but of different dimensionality ($\in \set{3, 5,10,20}$). In Figure~\ref{fig:objective_versus_iter}, we plot the dictionary learning objective against the iterations. As is expected, smaller data dimensionality shows faster convergence. That said, even 20-dimensional data was found to converge in less than 50 alternating iterations of the algorithm, which is remarkable.

\begin{figure*}%
\centering
\subfigure[]{\label{fig:time_against_dim}\includegraphics[width=0.4\linewidth]{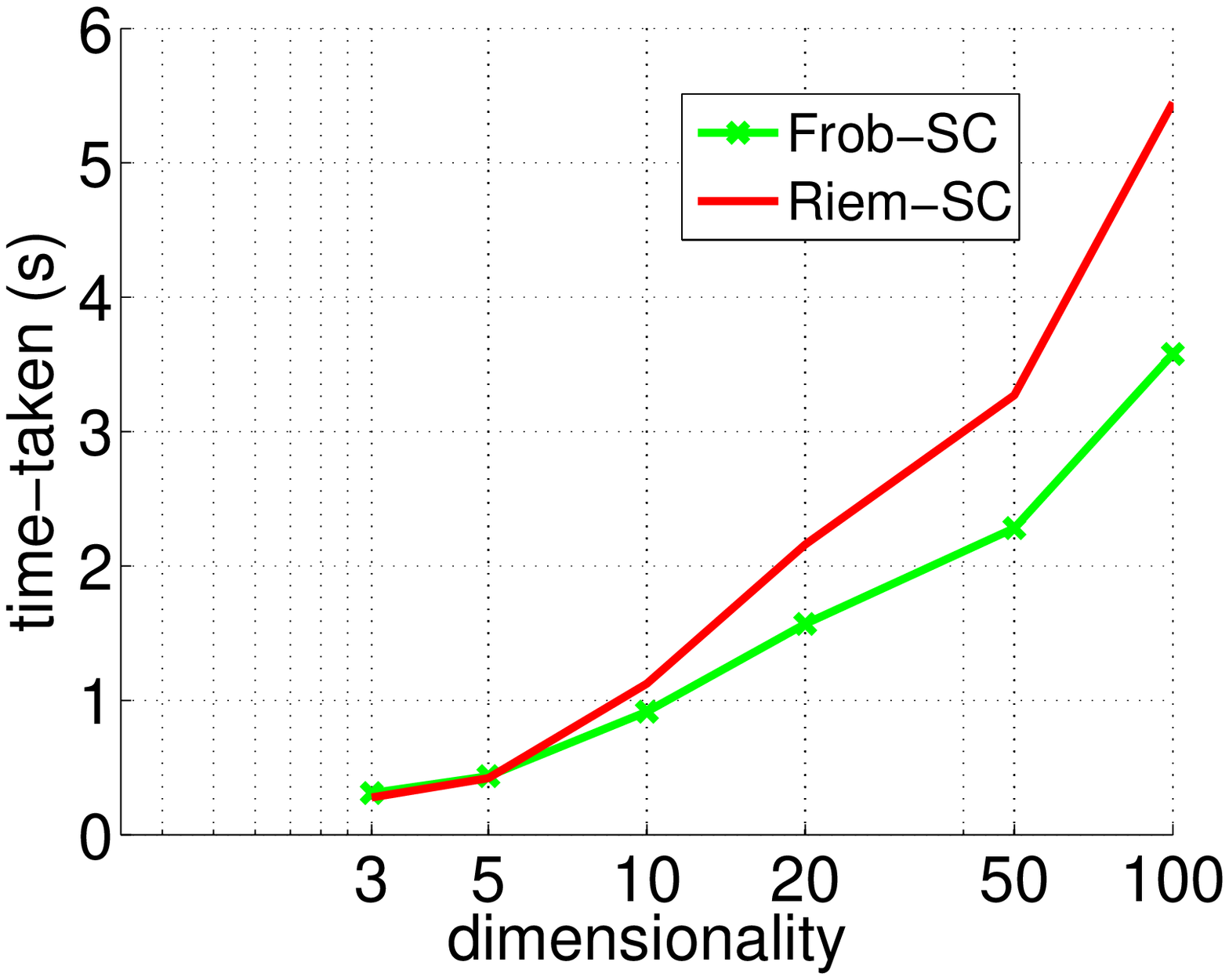}}
\subfigure[]{\label{fig:time_against_atoms}\includegraphics[width=0.4\linewidth]{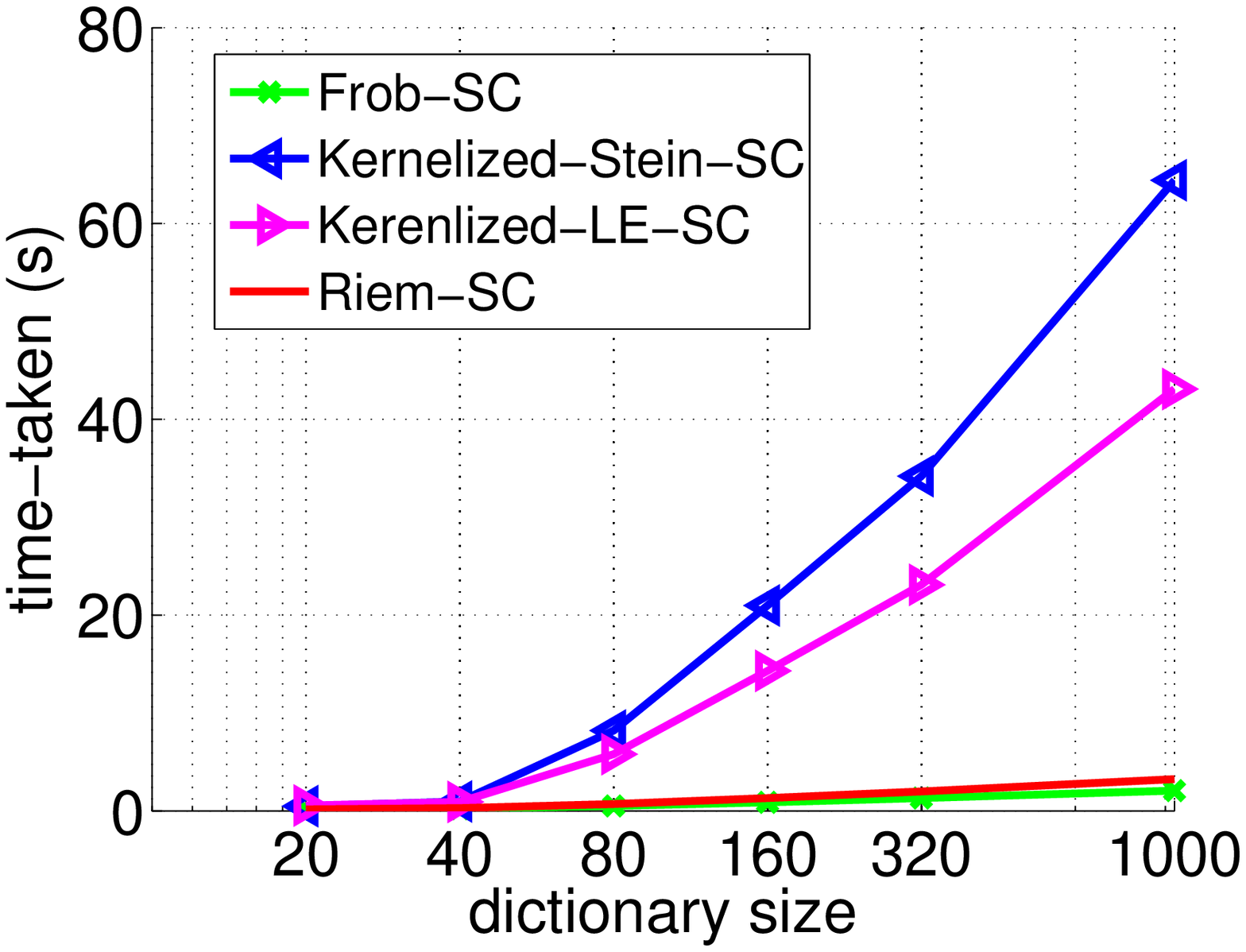}}
 \caption{Sparse coding time against (a) increasing matrix dimensionality and (b) increasing number of dictionary atoms. We used a maximum of 100 iterations for all the algorithms.}%
\label{fig:sim_results}%
\end{figure*}

\begin{figure*}%
\centering
\subfigure[]{\label{fig:sparsity_against_lambda}\includegraphics[width=0.4\linewidth]{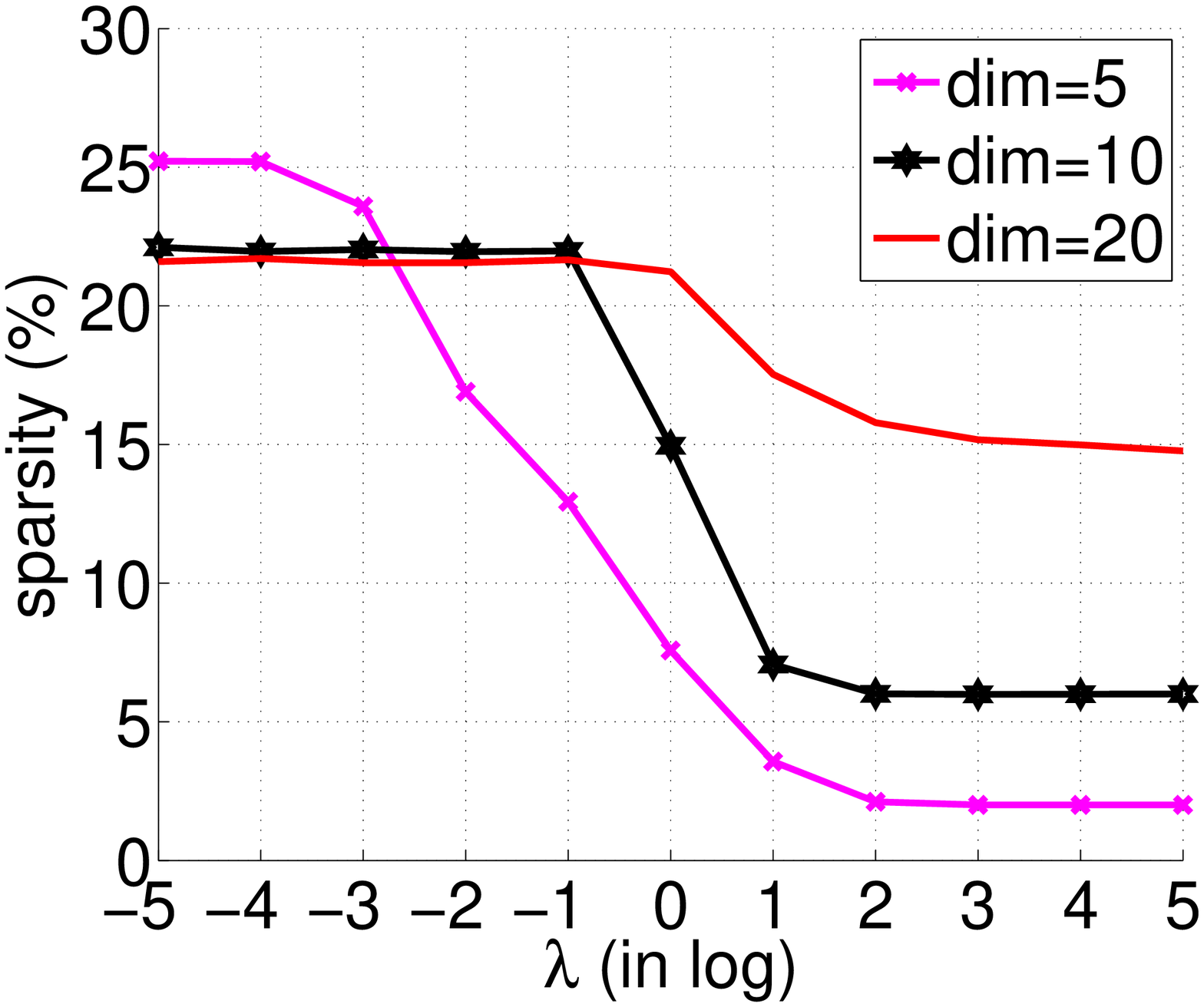}}
\subfigure[]{\label{fig:sparsity_against_lambda_compare_LE}\includegraphics[width=0.4\linewidth]{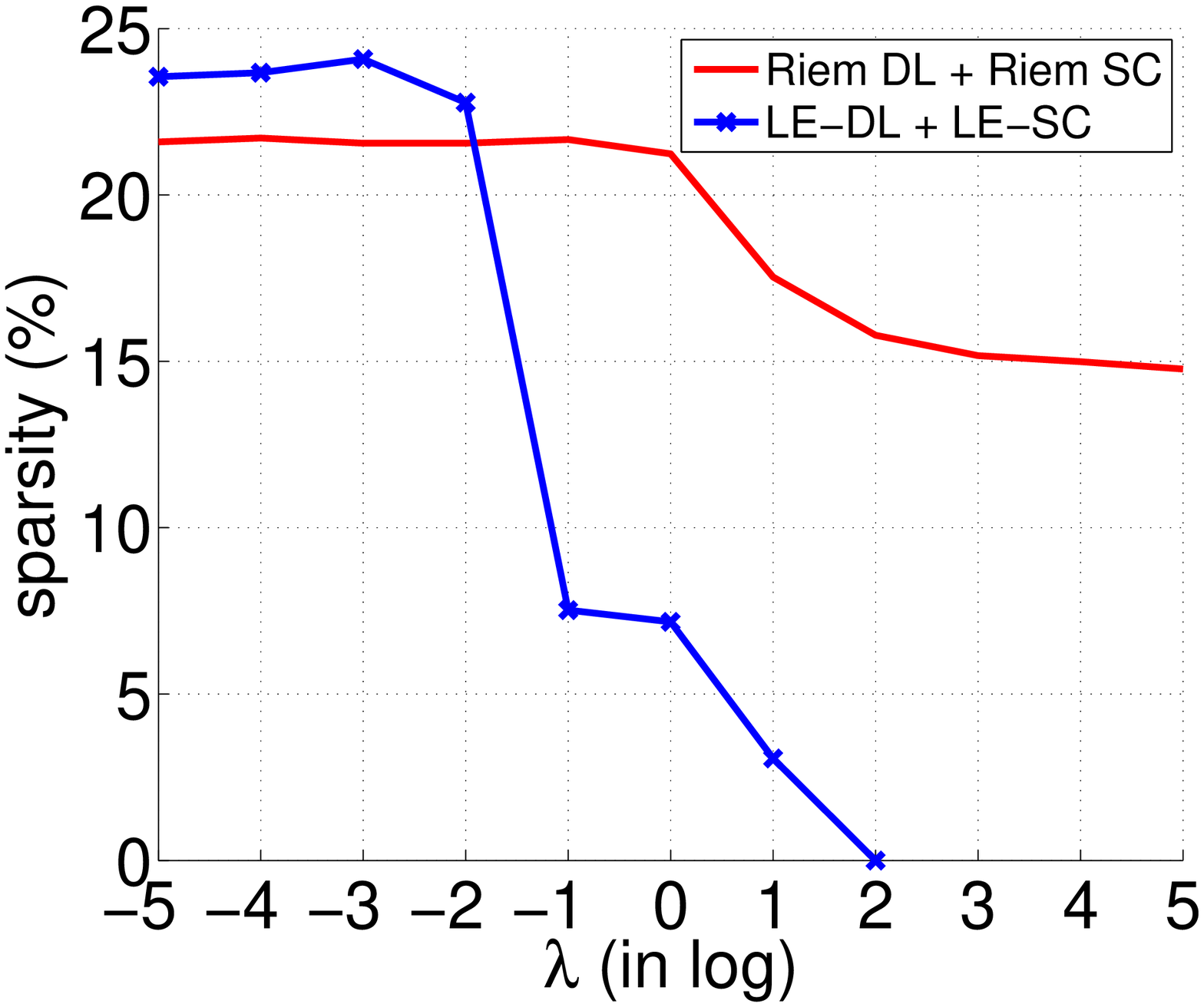}}
\caption{Sparsity of coding against (a) increasing sparsity inducing regularization $\lambda$ for various matrix dimensionality and (b) sparsity against lambda in comparison to that for log-Euclidean DL.}%
\label{fig:sparsity}%
\end{figure*}


\subsection{Experiments with Public Datasets}
Now let us evaluate the performance of our framework on computer vision datasets. We experimented on data available from four standard computer vision applications, namely  (i) 3D object recognition on the RGBD objects dataset~\cite{lai2011large}, (ii) texture recognition on the standard Brodatz dataset~\cite{ojala1996comparative}, (iii) person re-identification on the ETHZ people dataset~\cite{ess2007depth}, and (iv) face recognition on the Youtube faces dataset~\cite{youtubefaces}. We describe these datasets below.

\begin{figure*}[htbp]
	\centering		
		\subfigure[RGBD objects]{\includegraphics[trim={3 0 0 0},clip,height=2cm,width=4cm]{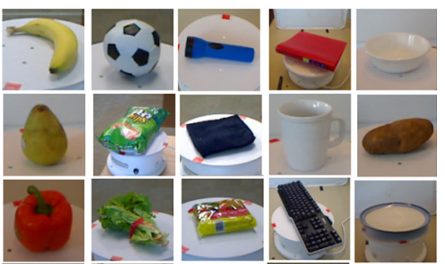}}
	  \subfigure[Brodatz textures]{\includegraphics[height=2cm,width=4cm]{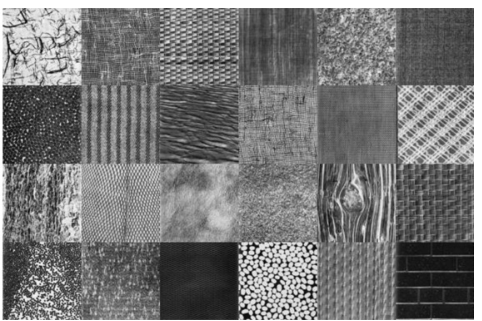}}
	  \subfigure[ETHZ appearances]{\includegraphics[height=2cm,width=4cm]{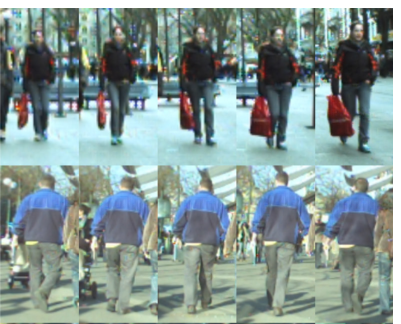}}
	  \subfigure[Youtube faces]{\includegraphics[height=2cm,width=4cm]{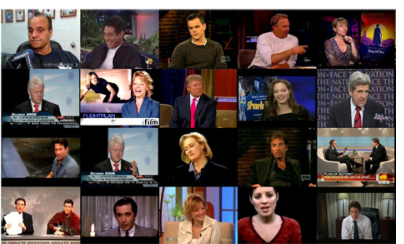}}
		\caption{Montage of sample images from the four datasets used in our experiments. From top, samples from the RGBD object recognition dataset, Brodatz texture recognition, ETHZ people re-identification dataset, and Youtube face recognition dataset.}	  
		\label{fig:allsampleimages}		
\end{figure*}

\subsubsection{Brodatz Texture}
\label{sec:brodatz}
Texture recognition is one of the most successful applications of covariance descriptors~\cite{de2008texture,porikli1}. For this evaluation, we used the Brodatz texture dataset\footnote{\url{http://www.ux.uis.no/~tranden/brodatz.html}}, from which we took 100 gray scale texture images, each of dimension $512\times 512$. We extracted $32\times 32$ patches from a dense grid without overlap thus generating 256 texture patches per image, and totalling 25600 patches in our dataset. To generate covariance descriptors from each patch, we followed the traditional protocol, i.e., we extracted a 5-dimensional feature descriptor from each pixel location in each patch. The features are given by: $F_{textures} = \left[x, y, I, abs(I_x), abs(I_y)\right]^T$, where the first two dimensions are the coordinates of a pixel from the top-left corner of a patch, the last three dimensions are the image intensity, and image gradients in  the $x$ and $y$ directions respectively. Region covariances of size $5\times 5$ are computed from all features in a patch.

\subsubsection{ETHZ Person Re-identification Dataset}
Tracking and identifying people in severely dynamic environments from multiple cameras play an important role in visual surveillance. The visual appearances of people in such applications are often noisy, and low resolution. Further, the appearances undergo drastic variations with  respect to their pose, scene illumination, and occlusions. Lately, covariance descriptors have been found to provide a robust setup for this task \cite{harandi2012sparse,ma2012bicov}. In this experiment, we evaluate the performance of clustering people appearances on the benchmark ETHZ dataset~\cite{schwartz09d}. This dataset consists of low-resolution images of tracked people from a real-world surveillance setup. The images are from 146 different individuals. There are about 5--356 images per person. Sample images from this dataset are shown in Figure~\ref{fig:allsampleimages}. There are a total of 8580 images in this dataset.

\begin{figure}
\centering
\includegraphics[width=6cm]{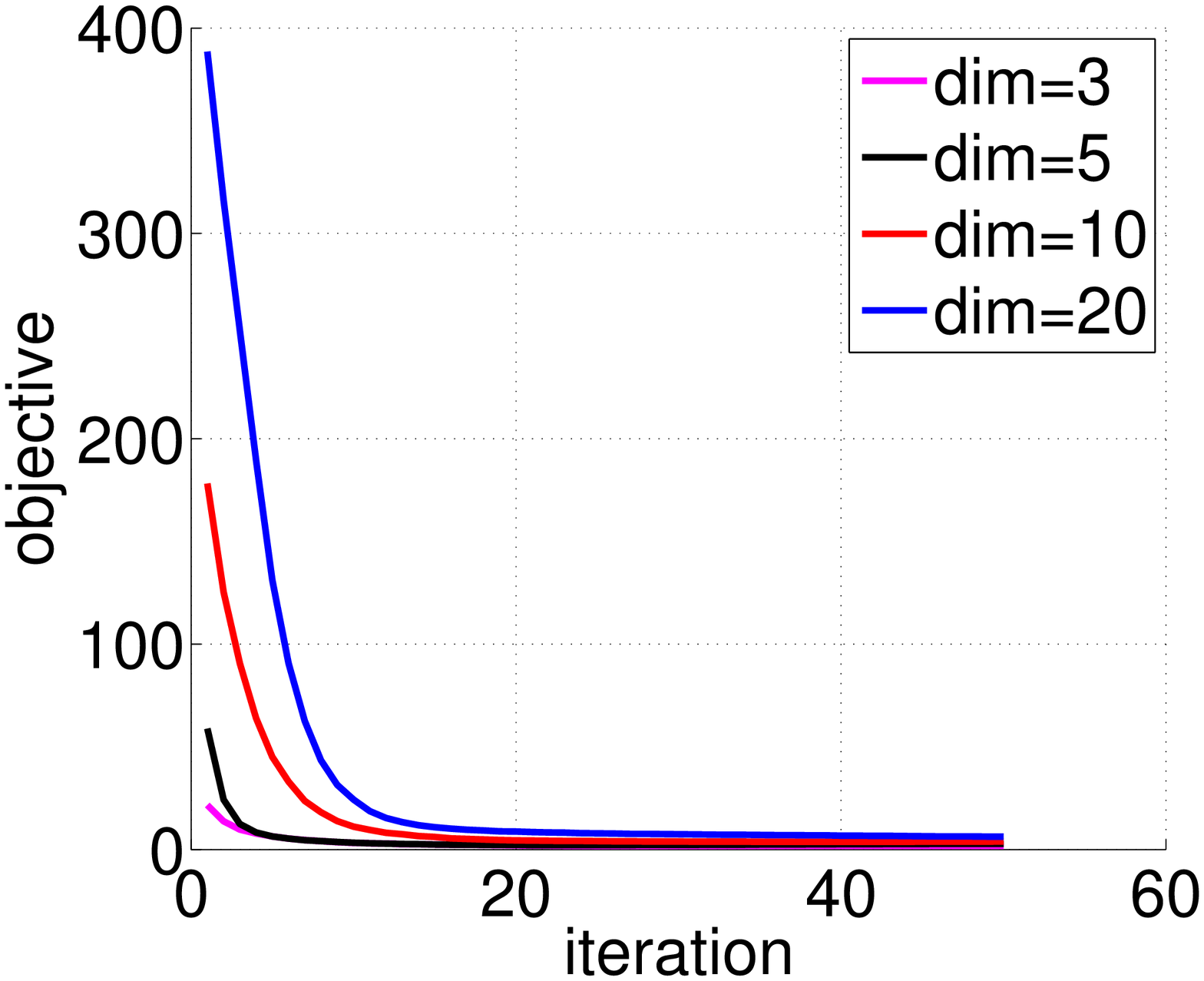}
\caption{Dictionary Learning objective using Riemannian conjugate gradient descent against increasing number of iterations (alternating with the sparse coding sub-problem). We plot the convergence of the objective for various dimensionality of the data matrices.}%
\label{fig:objective_versus_iter}
\end{figure}

\begin{figure*}[htbp]
	\centering		
	\subfigure[Brodatz textures]{\label{fig:Recall_K_clus_textures}\includegraphics[width=0.4\linewidth]{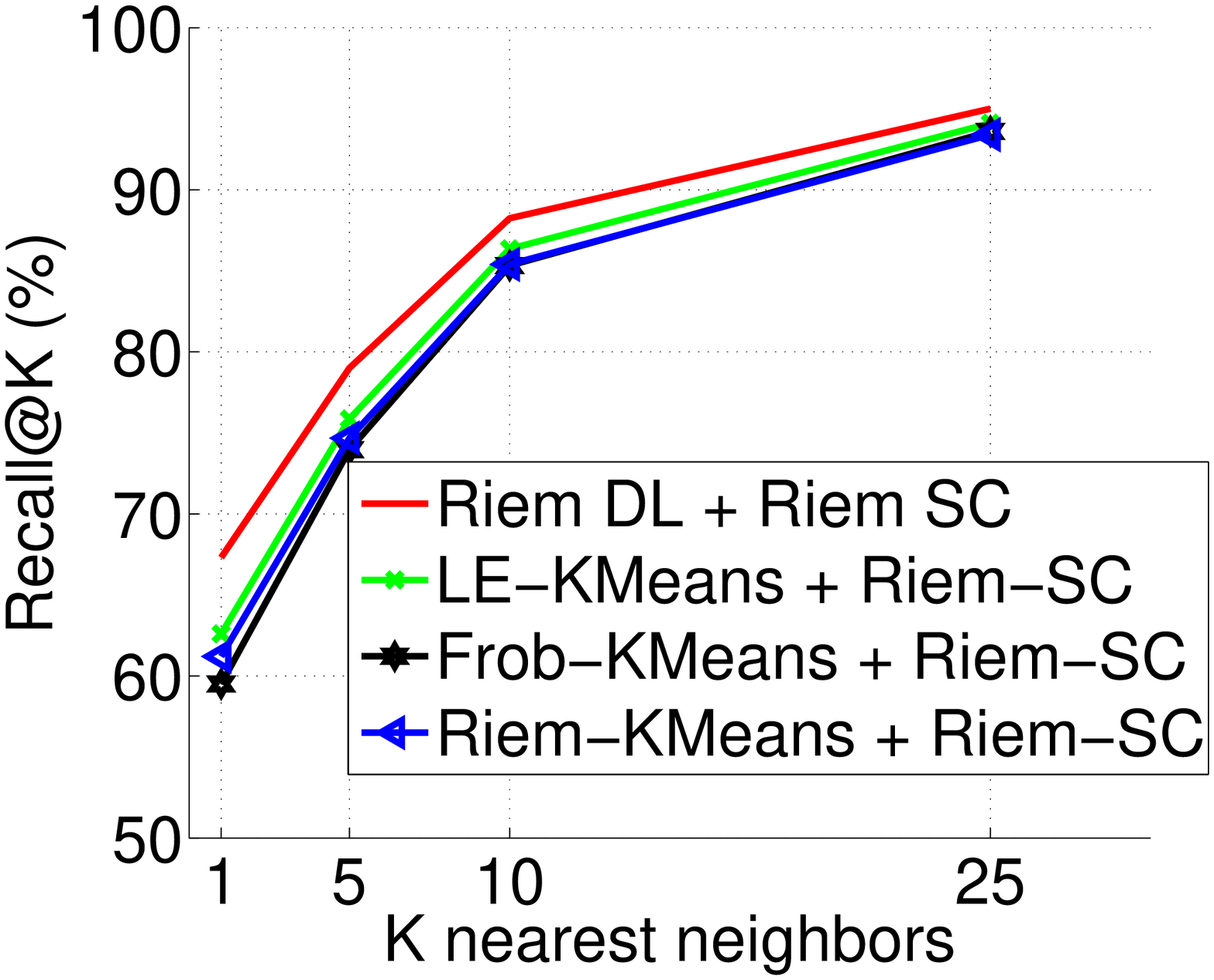}}
	\subfigure[RGB 3D objects]{\label{fig:Recall_K_clus_rgbd}\includegraphics[width=0.4\linewidth]{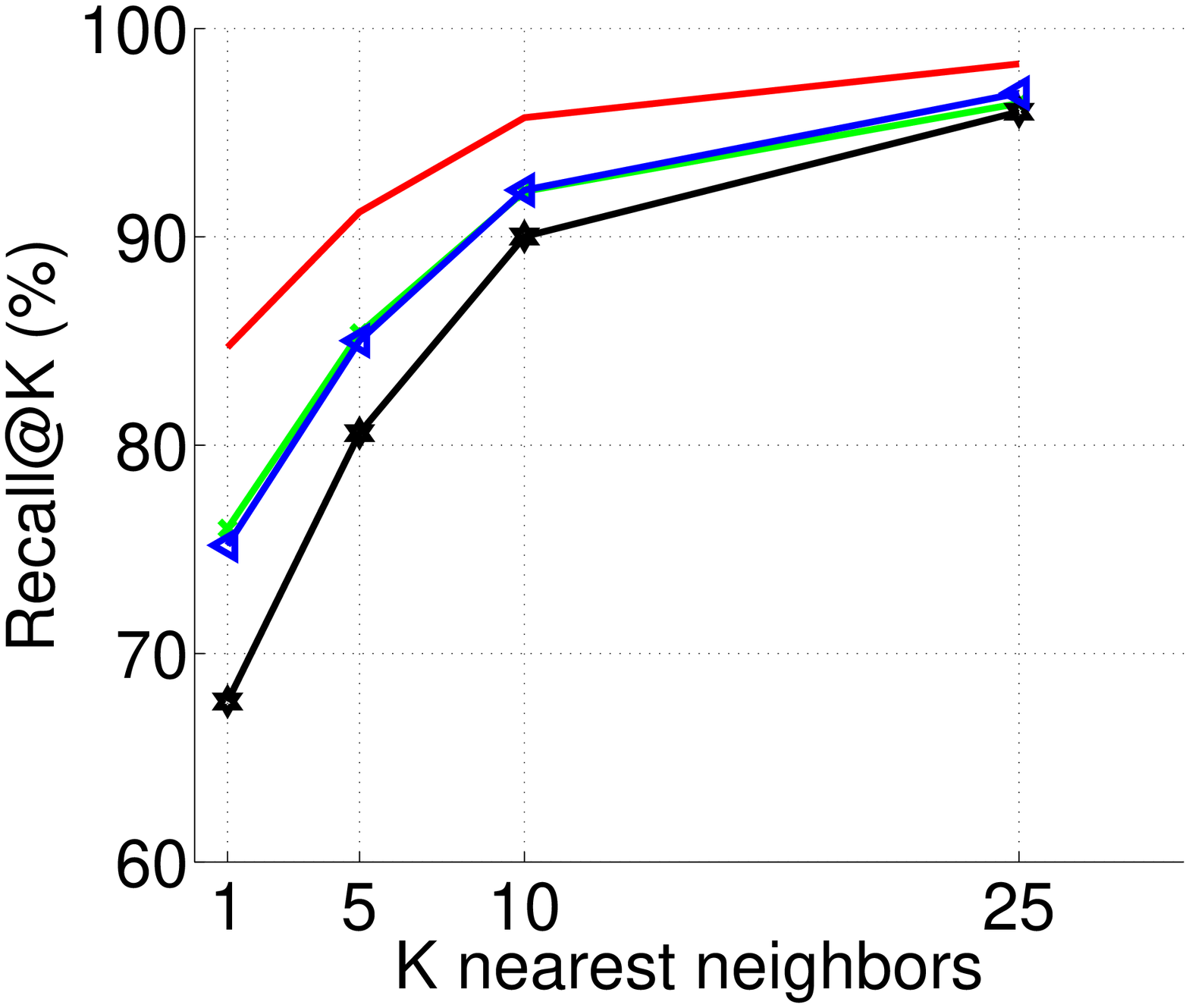}}\\
	\subfigure[ETHZ people]{\label{fig:Recall_K_clus_ethz}\includegraphics[width=0.4\linewidth]{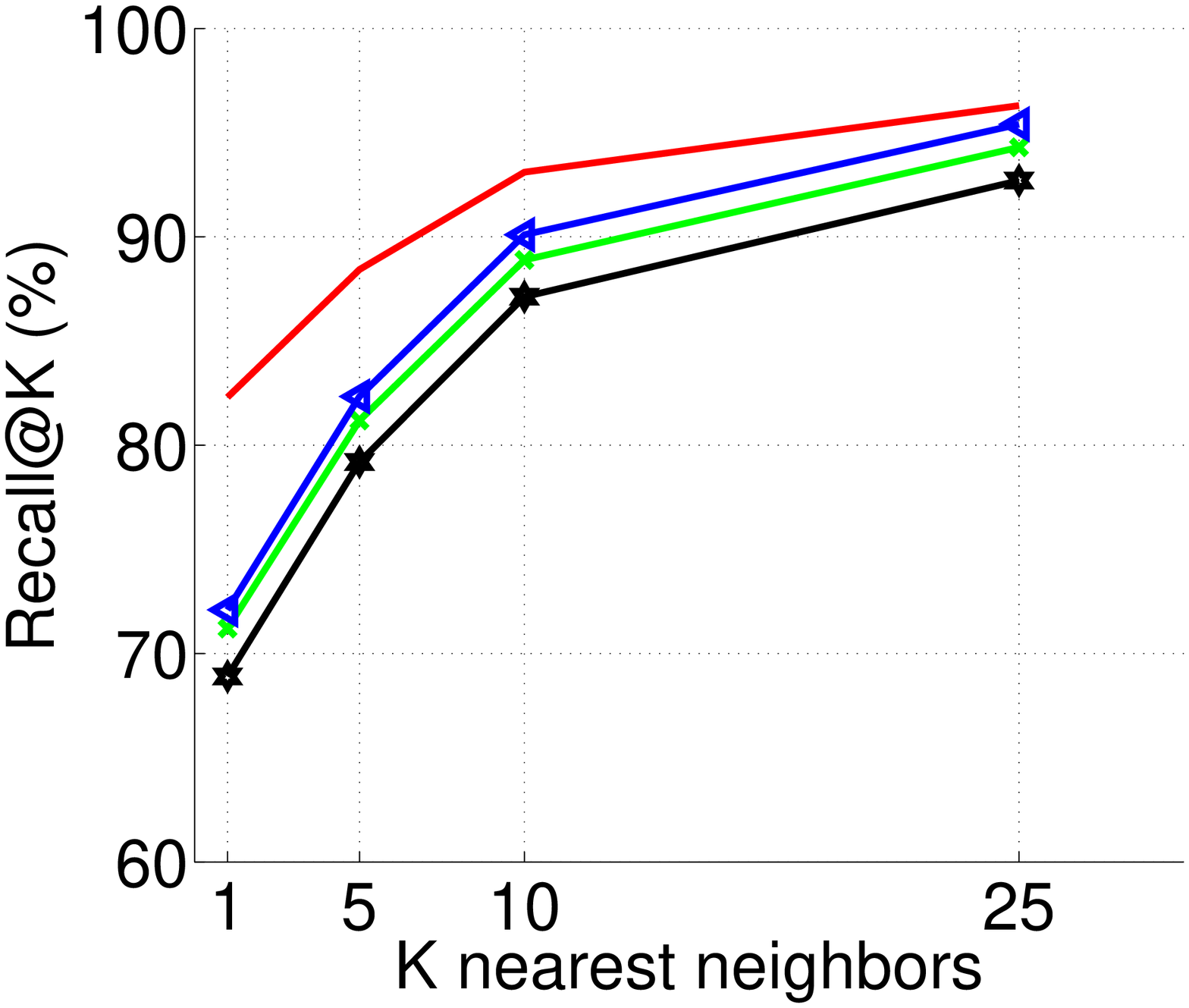}}
	\subfigure[Youtube faces]{\label{fig:Recall_K_clus_YT}\includegraphics[width=0.4\linewidth]{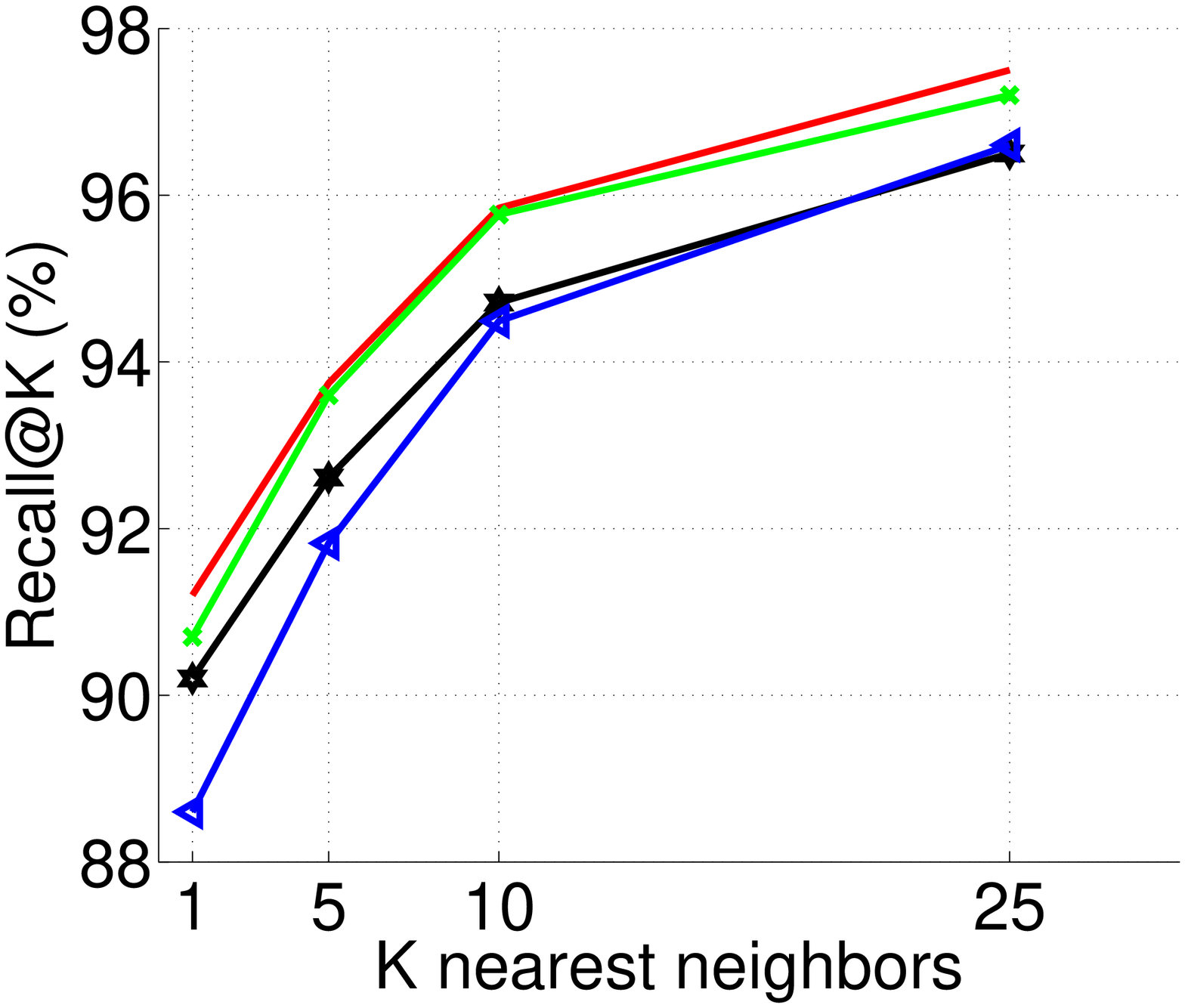}}	
		\caption{Results of Nearest neighbor recall@K accuracy against increasing number of retrieved points (K). Comparisons of Riem-DL and Riem-SC against other DL learning schemes based on clustering, while using our Riem-SC for sparse coding.}	  
		\label{fig:Recall@K_otherclustering}		
\end{figure*}

\textcolor{blue}{
\begin{table*}[!htb] 	
	\centering
	\parbox{.48\linewidth}
	{	
	\centering
	\begin{tabular}{cc}	
	Method & Accuracy (\%)\\
	\hline
	Riem-DL + Riem-SC        			& 74.9 \\
	LE-DL + LE-SC      						& 73.4 \\
	Frob-DL + Frob SC   					&  23.5 \\
	Kernelized LE DLSC~\cite{lilog}+~\cite{harandi2009riemannian}      & 47.9 \\	
	Kernelized Stein DLSC ~\cite{harandi2012sparse}+~\cite{harandi2009riemannian}	& $\textbf{76.7}$ \\
	Tensor DL+SC~\cite{sivalingam2015tensor}  & 37.1\\
	GDL~\cite{sra2011generalized} & 47.7 \\
	Random DL + Riem-SC 					& 70.3 \\
	\hline
	\end{tabular}
	\caption{Brodatz texture dataset}
	\label{tab:brodatz_results}
	}
	\hfill\vspace{4pt}
	\parbox{.48\linewidth}
	{
	\centering
	\begin{tabular}{cc}	
	Method & Accuracy (\%)\\
	\hline
	Riem-DL + Riem-SC        & 80.0 \\
	LE-DL + LE-SC      			 & 80.5 \\
	Frob-DL + Frob SC   		 & 54.2 \\
	Kernelized LE DLSC~\cite{lilog}+~\cite{harandi2009riemannian}      & \textbf{86.0}\\
	Kernelized Stein DLSC~\cite{harandi2012sparse}+~\cite{harandi2009riemannian}	& 85.7\\
	Tensor DL+SC~\cite{sivalingam2015tensor}          & 68.1 \\
	GDL~\cite{sra2011generalized} & 43.0 \\
	Random DL + Riem-SC 		&  62 \\
	\hline
	\end{tabular}	
	\caption{RGBD objects}
	\label{tab:rgbd_results}
	}			
%
	\parbox{.48\linewidth}
	{	
	\centering
	\begin{tabular}{cc}	
	Method & Accuracy (\%)\\
	\hline
	Riem-DL + Riem-SC        & 80.5	\\
	LE-DL + LE-SC      			 & 80.0 \\
	Frob-DL + Frob SC   		 & 77.6 \\
	Kernelized LE DLSC~\cite{lilog}+~\cite{harandi2009riemannian}     & \textbf{86.6} \\
	Kernelized Stein DLSC ~\cite{harandi2012sparse}~\cite{harandi2009riemannian}	& 71.6 \\
	Tensor DL+SC~\cite{sivalingam2015tensor}          & 67.4\\
	GDL~\cite{sra2011generalized} & 71.0\\
	Random DL + Riem-SC & 54.6 \\
	\hline
	\end{tabular}
	\caption{ETHZ People dataset}
	\label{tab:ethz_results}
	}
	\hfill\vspace{4pt}
	\parbox{.48\linewidth}
	{
	\centering
	\begin{tabular}{cc}	
	Method & Accuracy (\%)\\
	\hline
	Riem-DL + Riem-SC        & 92.4 \\
	LE-DL + LE-SC      			 & 82.6 \\
	Frob-DL + Frob SC   		 &  82.9 \\
	Kernelized LE DSC~\cite{lilog}+~\cite{harandi2009riemannian}   &  \textbf{93.1}\\
	Kernelized Stein DLSC~\cite{harandi2012sparse} +~\cite{harandi2009riemannian}& 70.1\\	
	GDL~\cite{sra2011generalized} & 92.0 \\
	Random DL + Riem-SC 		 &  83.9 \\
	\hline
	\end{tabular}	
	\caption{Youtube Faces Dataset}
	\label{tab:yt_results}
	}				
	{TABLES: Comparison of classification accuracy (using a linear SVM and one-against-all classification) with sparse coding when the dictionary is learned using the respective DL method. The standard deviation was less than 5\% for all methods.}
\end{table*}
}

Our goal in this experiment is to evaluate the performance of our DLSC framework to learn generic dictionaries on covariance descriptors produced from this application. Note that some of the classes in this dataset does not have enough instances to learn a specific dictionary for them. Several types of features have been suggested in literature for generating covariances on this dataset that have shown varying degrees of success such as Gabor wavelet based features~\cite{ma2012bicov}, color gradient based features~\cite{harandi2012sparse}, etc. Rather than detailing the results on several feature combinations, we describe here the feature combination that worked the best in our experiments. For this purpose, we used a validation set of 500 covariances and 10 true clusters from this dataset. The performance was evaluated using the Log-Euclidean SC setup with a dictionary learning via Log-Euclidean K-Means. We used a combination of nine features for each image as described below:
\begin{align}
F_{ETHZ} &= \left[x\ I_r\ I_g\ I_b\ Y_i\  \abs(I_x)\ \abs(I_y)\right.\\\nonumber
					&\quad\quad \left.\abs(\sin(\theta)+\cos(\theta))\ \abs(H_y)\right],
\end{align}
where $x$ is the x-coordinate of a pixel location, $I_r,I_g,I_b$ are the RGB color of a pixel, $Y_i$ is the pixel intensity in the YCbCr color space, $I_x,I_y$ are the gray scale pixel gradients, and $H_y$ is the y-gradient of pixel hue. Further, we also use the gradient angle $\theta=tan^{-1}(I_y/I_x)$ in our feature set. Each image is resized to a fixed size $300\times 100$, and is divided into upper and lower parts. We compute two different region covariances for each part, which are combined as two block diagonal matrices to form a single covariance descriptor of size $18\times 18$ for each appearance image.

\subsubsection{3D Object Recognition Dataset}
The goal of this experiment is to recognize objects in 3D point clouds. To this end, we used the public RGB-D Object dataset~\cite{lai2011large}, which consists of about 300 objects belonging to 51 categories and spread in about 250K frames. We used approximately 15K frames for our evaluation with approximately 250-350 frames devoted to every object seen from three different view points (30, 45, and 60 degrees above the horizon). Following the procedure suggested in~\cite{fehr2013covariance}[Chap.~5], for every frame, the object was segmented out and 18 dimensional feature vectors generated for every 3D point in the cloud (and thus $18\times 18$ covariance descriptors); the features we used are as follows:
\begin{align}
F_{RGBD} &= \left[x,y,z, I_r, I_g, I_b, I_x,I_y, I_{xx}, I_{yy}, I_{xy},\right. \nonumber\\
& \left.I_m, \delta_x, \delta_y,\delta_{m},\nu_x, \nu_y, \nu_z\right],
\end{align}
where the first three dimensions are the spatial coordinates, $I_m$ is the magnitude of the intensity gradient, $\delta$'s represent gradients over the depth-maps, and $\nu$ represents the surface normal at the given 3D point. 

\subsubsection{Youtube Faces Dataset}
In this experiment, we evaluate the performance of the Riemannian DLSC setup to deal with a larger dataset of high-dimensional covariance descriptors for face recognition. To this end, we used the challenging Youtube faces dataset~\cite{youtubefaces} that consists of 3425 short video clips of 1595 individuals, each clip containing between 48--6K frames. There are significant variations in head pose, context, etc. for each person across clips and our goal is to associate a face with its ground truth person label. We proceed by first cropping out face regions from the frames by applying a state-of-the-art face detector~\cite{zhu2012face}, which results in approximately 196K face instances. As most of the faces within a clip do not have significant variations, we subsample this set randomly to generate our dataset of $\sim$43K face patches. Next, we convolved the image with a filter bank of 40 Gabor filters with 5 scales and 8 different orientations to extract the facial features for each pixel, generating $40\times 40$ covariances.

\begin{figure*}[htbp]
	\centering	
	\subfigure[Brodatz textures]{\label{fig:Recall_K_otherDL_textures}\includegraphics[width=0.4\linewidth]{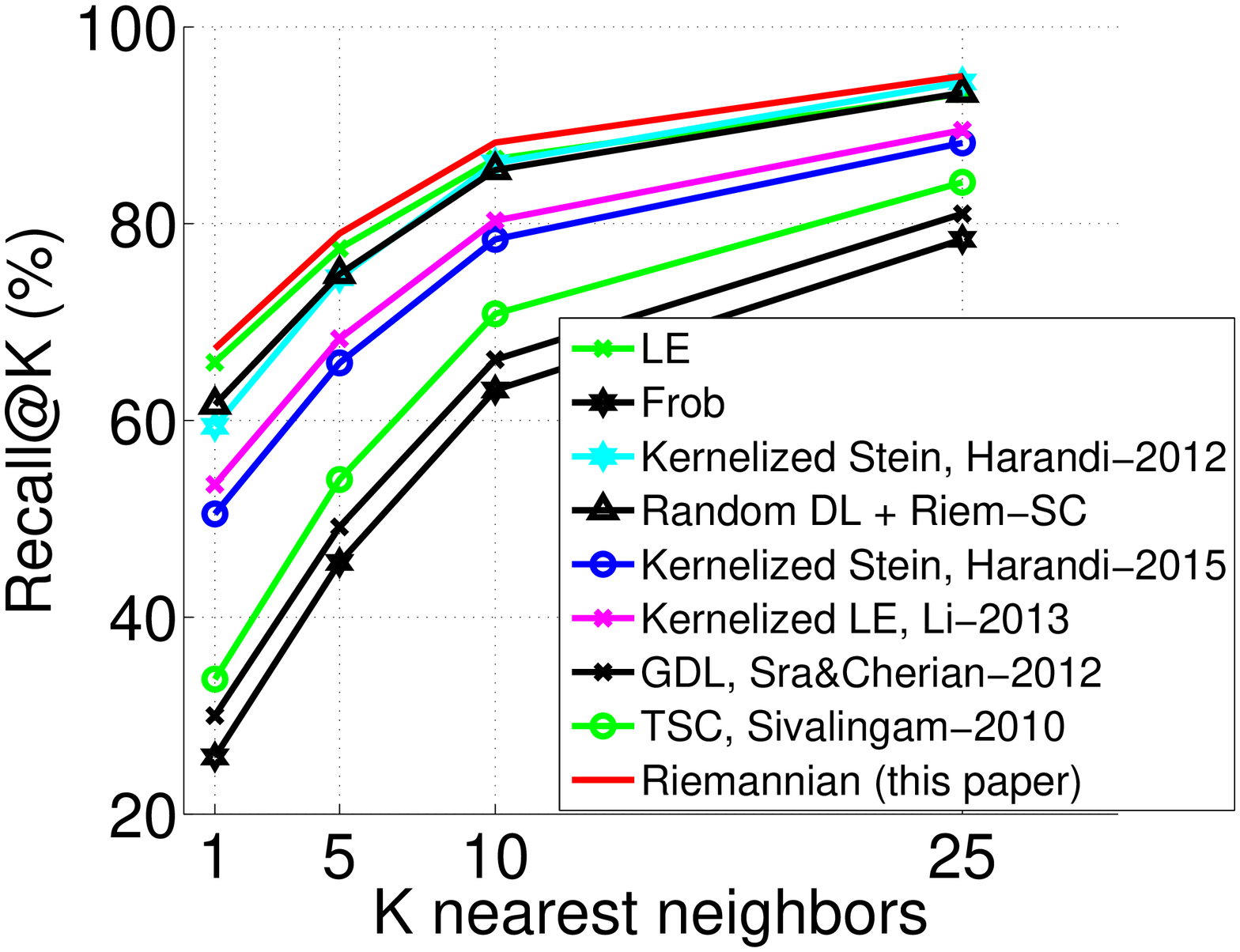}}
	\subfigure[RGB 3D objects]{\label{fig:Recall_K_otherDL_rgbd}\includegraphics[width=0.4\linewidth]{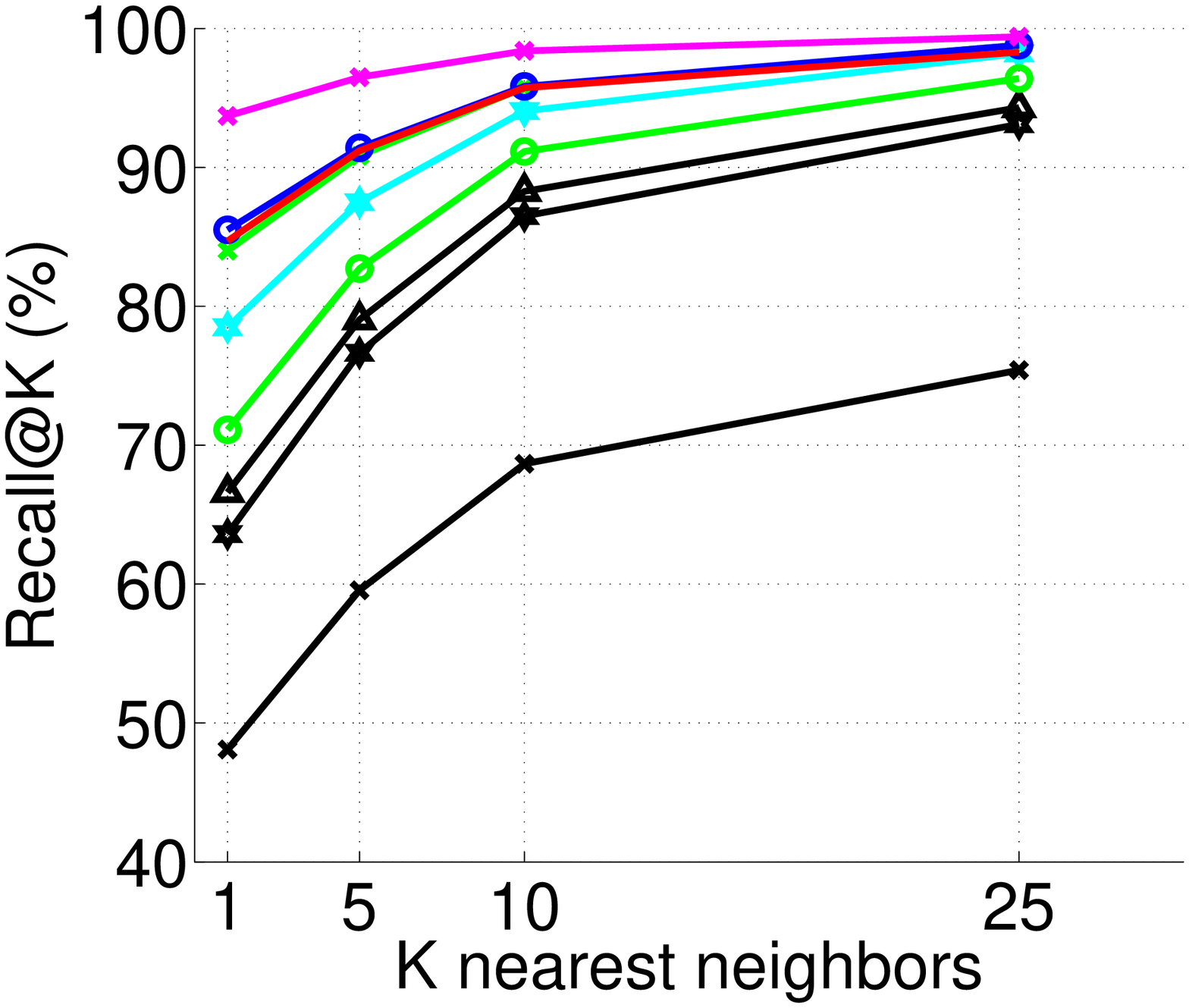}}\\
	\subfigure[ETHZ people]{\label{fig:Recall_K_otherDL_ethz}\includegraphics[width=0.4\linewidth]{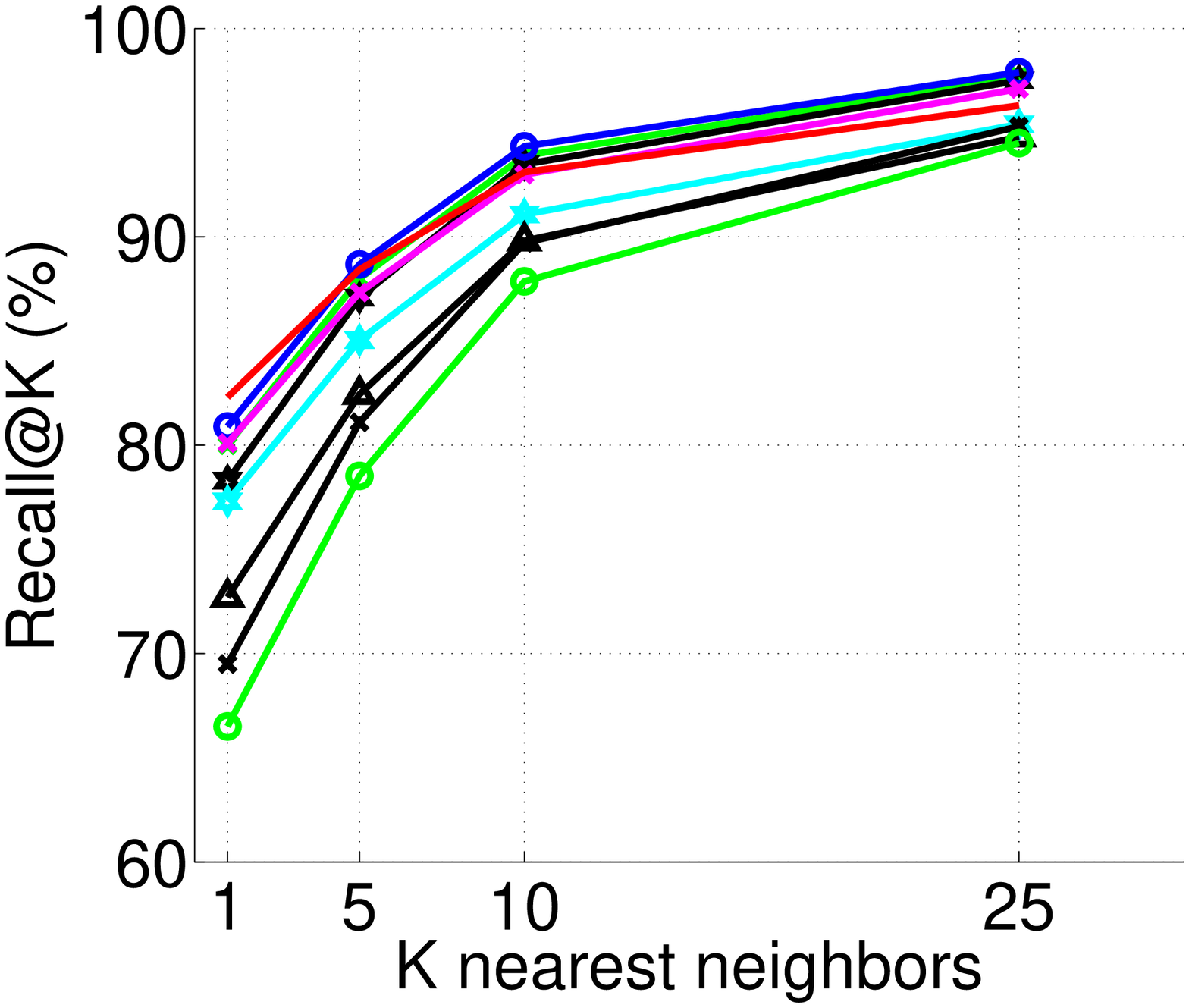}}	
	\subfigure[Youtube faces]{\label{fig:Recall_K_otherDL_YT}\includegraphics[width=0.4\linewidth]{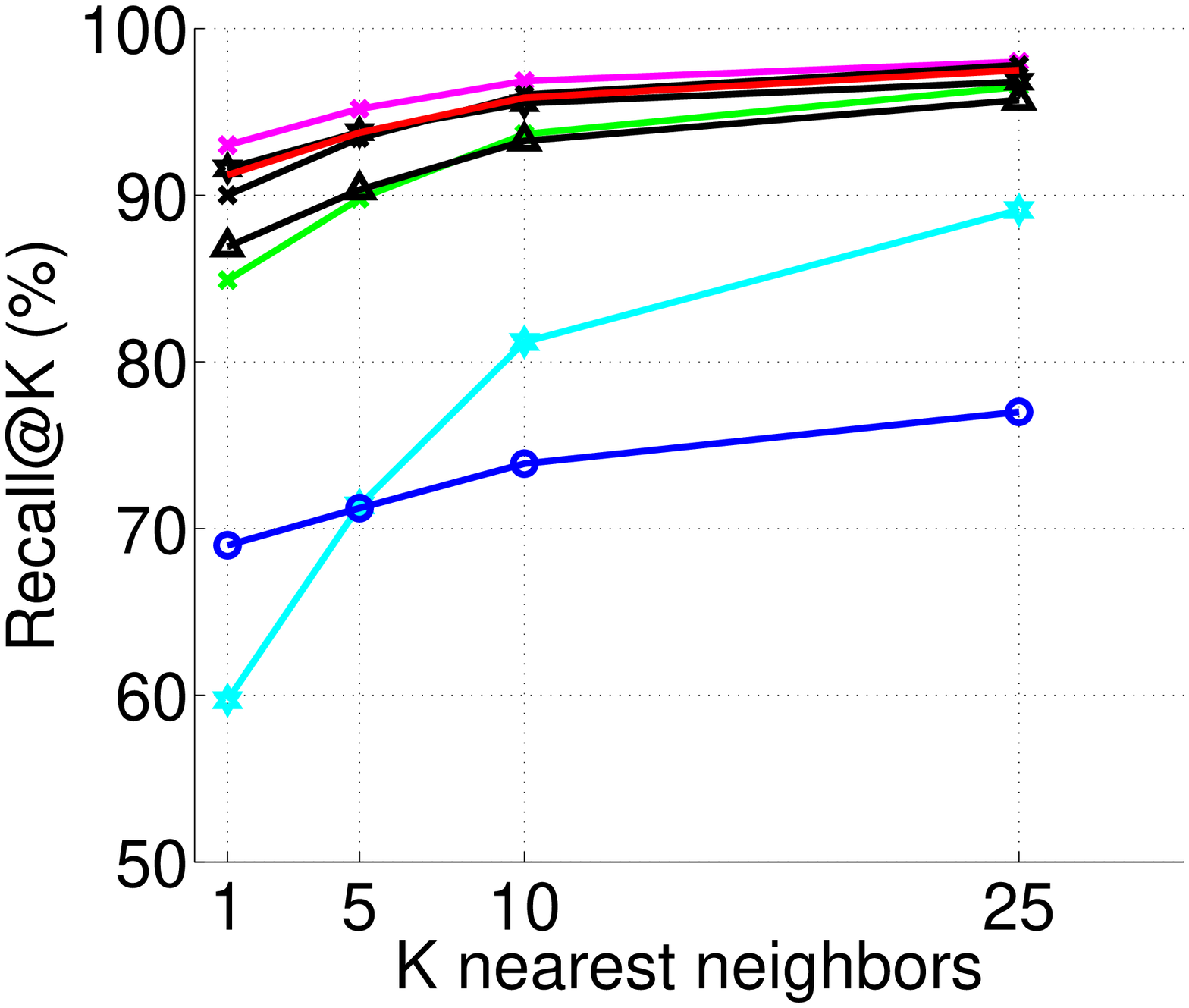}}
		\caption{Results of k-nearest neighbor retrieval accuracy against an increasing number of retrieved points (K). Comparisons of Riem-DL and Riem-SC against the state of the art dictionary learning and sparse coding combinations.}	  
		\label{fig:Recall_K_otherDL}		
\end{figure*}

\subsection{Experimental Setup}
\subsubsection{Evaluation Techniques}
We evaluated our algorithms from two perspectives, namely (i) nearest neighbor (NN) retrieval against a gallery set via computing the Euclidean distances between sparse codes, and (ii) one-against-all classification using a linear SVM trained over the sparse codes. Given that computing the geodesic distance between SPD matrices is expensive, while the Frobenius distance between them results in poor accuracy, the goal of the first experiment is to evaluate the quality of sparse coding to approximate the input data in terms of codes that belong to the non-negative orthant of the Euclidean space -- superior performance implying that the sparse codes provide efficient representations that could bypass the Riemannian geometry, and can enable other faster indexing schemes such as locality sensitive hashing for faster retrieval. Our second experiment evaluates the linearity of the space of sparse codes -- note that they are much higher dimensional than the original covariances themselves and thus we expect them to be linearly separable in the sparse space. 

\subsubsection{Evaluation Metric} 
For classification experiments, we use the one-against-all classification accuracy as the evaluation metric. For NN retrieval experiments, we use the Recall@K accuracy, which is defined as follows. Given a gallery $\dataset$ and a query set $\mathcal{Q}$. Recall@K computes the average accuracy when retrieving $K$ nearest neighbors from $\dataset$ for each instance in $\mathcal{Q}$. Suppose $G_K^q$ stands for the set of ground truth class labels associated with the $q$th query, and if $S^q_K$ denotes the set of labels associated with the $K$ neighbors found by some algorithm for the $q$ queries, then
\begin{equation}
	Recall@K = \frac{1}{|Q|} \sum_{q\in Q}\frac{|G_K^q\cap S_K^q|}{|G^q_K|}.
	\label{eq:accuracy}
\end{equation} 

\subsubsection{Data Split}
All the experiments used 5-fold cross-validation in which 80\% of the datasets were used for training the dictionary, $10\%$ for generating the gallery set or as training set for the linear SVM, and the rest as the test/query points. We evaluate three setups for generating the dictionaries, (i) using a proper dictionary learning strategy, and (ii) using clustering the training set via K-Means using the appropriate distance metric, and (iii) random sampling of the training set. 

\subsubsection{Hyperparameters}
The size of the dictionary was considered to be twice the number of classes in the respective dataset. This scheme was considered for all the comparison methods as well. We experimented with larger sizes, but found that performance generally almost saturates. This is perhaps because the datasets that we use already have a large number of classes, and thus the dictionary sizes generated using this heuristic makes them already significantly overcomplete. The other hyperparameter in our setup is the sparsity of the generated codes. \textcolor{blue}{As the different sparse coding methods (including ours and the methods that we compare to) have varied sensitivity to the regularization parameter, comparing all the methods to different sparsities turned out to be cumbersome. Thus, we decided to fix the sparsity of all methods to 10\%-sparse and adjusted the regularization parameter for each method appropriately (on a small validation set separate from the training set). To this end, we used $\lambda_B = 0.1$ for textures, $10$ for the ETHZ and RGBD datasets, and 100 for the faces dataset. For the faces dataset, we found it to be difficult to attain the desired sparsity by tuning the regularization parameter. Thus, we used a regularization of 100 and selected the top 10\% sparse coefficients.}

\subsubsection{Implementation Details}
Our DLSC scheme was implemented in MATLAB. We used the ManOpt Riemannian geometric optimization toolbox~\cite{boumal2014manopt} for implementing the CG method in our DL sub-problem. As our problem is non-convex, we found that initializing the dictionary learning setup using K-Means clustering	(using the Karcher mean algorithm~\cite{pennec2006riemannian}) demonstrate faster convergence.

\begin{table*}[!htb] 
	\centering
	\parbox{.48\linewidth}
	{	
	\centering
	\begin{tabular}{cc}	
	Method & Accuracy (\%)\\
	\hline
	Riem-DL + Riem-SC        & \textbf{74.9} \\
	LE-KMeans + Riem-SC      & 70.0 \\
	Frob-KMeans + Riem-SC    & 66.5 \\
	Riem-KMeans +Riem-SC     & 70.0 \\
	\hline
	\end{tabular}
	\caption{Brodatz texture dataset}
	\label{tab:brodatz_results_clustering}
	}
	\hfill\vspace{4pt}
	\parbox{.48\linewidth}
	{
	\centering
	\begin{tabular}{cc}	
	Method & Accuracy (\%)\\
	\hline
	Riem-DL + Riem-SC        & \textbf{80.0} \\
	LE-KMeans + Riem-SC      & 66.2 \\
	Frob-KMeans + Riem-SC    & 61.1 \\
	Riem-KMeans +Riem-SC     & 67.8 \\
	\hline
	\end{tabular}	
	\caption{RGBD objects}
	\label{tab:rgbd_results_clustering}
	}			
\end{table*}
\begin{table*}[!htb] 
	\centering
	\parbox{.48\linewidth}
	{	
	\centering
	\begin{tabular}{cc}	
	Method & Accuracy (\%)\\
	\hline
	Riem-DL + Riem-SC        & \textbf{80.5}	\\
	LE-KMeans + Riem-SC      & 54.9 \\
	Frob-KMeans + Riem-SC    & 55.5 \\
	Riem-KMeans +Riem-SC     & 57.5 \\
	\hline
	\end{tabular}
	\caption{ETHZ People dataset}
	\label{tab:ethz_results_clustering}
	}
	\hfill\vspace{4pt}
	\parbox{.48\linewidth}
	{
	\centering
	\begin{tabular}{cc}	
	Method & Accuracy (\%)\\
	\hline
	Riem-DL + Riem-SC        & \textbf{92.4} \\
	LE-KMeans + Riem-SC      & 87.1 \\
	Frob-KMeans + Riem-SC    & 88.7 \\
	Riem-KMeans +Riem-SC     & 85.8 \\
	\hline
	\end{tabular}	
	\caption{Youtube Faces Dataset}
	\label{tab:yt_results_clustering}
	}			
	{TABLES: Comparison of classification accuracy (using a linear SVM and one-against-all classification) using Riemannian sparse coding (Riem-SC)
while the dictionary atoms are taken as the centroids of K-Means clusters. The standard deviation was less than 8\% for all methods.}
\end{table*}

\subsection{Results}
In this section, we compare the performance of our dictionary learning (Riem-DL) and sparse coding (Riem-SC) method against several state of the art DLSC schemes on the four datasets that we described above. Our choice of comparison methods include (i) Riemannian geometric methods such as log-Euclidean (LE-DL + LE-SC), \textcolor{blue}{(ii) Kernelized methods using the Stein kernel (Kernelized Stein) with the framework in~\cite{harandi2012sparse}, (iii) Kernelized Stein using the recent generic framework in~\cite{harandi2009riemannian}, (iv) Kernelized Log-Euclidean metric proposed in~\cite{lilog} but using the generic framework in~\cite{harandi2009riemannian}, (v) Euclidean DLSC (Frob-DL + Frob-SC), (vi) using a dictionary generated by random sampling the dataset followed by sparse coding using our Riemannian method (Random-DL + Riem-SC), (vii) using the tensor dictionary learning sparse coding (TSC) setup~\cite{sivalingam2015tensor}, and (viii) generalized dictionary learning~\cite{sra2011generalized}.} In Figure~\ref{fig:Recall_K_otherDL}, we show the performance comparison for the task of K-NN where $K$ is increased from one to 25. In Tables~\ref{tab:brodatz_results},~\ref{tab:rgbd_results},~\ref{tab:ethz_results}, and \ref{tab:yt_results}, we show the performance for the one-against-all classification setup. 


An alternative to dictionary learning that is commonly adopted is to approximate the dictionary by using the centroids of clusters generated from a K-Means clustering of the dataset. Such a method is faster in comparison to a Riemannian DL, while also demonstrate reasonable performance~\cite{cherian2014riemannian,sivalingam2010tensor}. Thus, an important experiment with regard to learning the dictionary is to make sure using dictionary learning provides superior performance compared to this ad hoc setup. In Figure~\ref{fig:Recall@K_otherclustering}, we plot the K-NN retrieval when we use a clustering scheme to generate the dictionary. In Tables~\ref{tab:brodatz_results_clustering},~\ref{tab:rgbd_results_clustering},~\ref{tab:ethz_results_clustering}, and \ref{tab:yt_results_clustering}, we show the same in a classification setup.

\subsection{Discussion of Results}
With regard to Figure~\ref{fig:Recall_K_otherDL}, we found that the performance of different methods is diverse across datasets. For example, the log-euclidean DLSC variant (LE-DL+LE-SC) is generally seen to showcase good performance across datasets. \textcolor{blue}{The kernelized DLSC methods (Kernelized Stein and Kernelized LE) demonstrate superior performance on almost all the datasets. The most surprising of the results that we found was for the Frob-DL case. It is generally assumed that using Frobenius distance for comparing SPD matrices leads to poor accuracy, which we see in Figures~\ref{fig:Recall_K_otherDL_textures},~\ref{fig:Recall_K_otherDL_rgbd}, and~\ref{fig:Recall_K_otherDL_ethz}. However, for the Youtube faces dataset, we found that the SPD matrices are poorly conditioned. As a result, taking the logarithm (as in the LE-DL scheme) of these matrices results in amplifying the influence of the smaller eigenvalues, which is essentially noise. When learning a dictionary, the atoms will be learned to reconstruct this noise against the signal, thus leading to inferior accuracy than for FrobDL or GDL which do not use matrix logarithm. We tried to circumvent this problem by tuning the small regularization that we add to the diagonal entries of these matrices, but that did not help.  Other older DLSC methods such as TSC are seen to be less accurate compared to recent methods. We could not run the TSC method on the faces dataset as it was found to be too slow to sparse code the larger covariances. In comparison to all the above methods, Riem-DL+Riem-SC was found to produce consistent, competitive (and sometimes better) performance, substantiating the usefulness of our proposed method. While running the experiments, we found that the initialization of our DL sub-problem (from K-Means) played an important role in achieving this superior performance. In Tables~\ref{tab:brodatz_results},~\ref{tab:rgbd_results},~\ref{tab:ethz_results}, and \ref{tab:yt_results} we show the results for classification using the sparse codes. The kernelized LE seems to be significantly better in this setting. However, our Riemannian scheme does demonstrate promise by being the second best in most of the datasets.}

The usefulness of our Riem-DL is further evaluated against alternative DL schemes via clustering in Figure~\ref{fig:Recall@K_otherclustering}. We see that learning the dictionary using Riem-DL demonstrates the best performance against the next best and efficient alternative of using the LE-KMeans that was done in~\cite{cherian2014riemannian}. Using Frob-KMeans or using a random dictionary are generally seen to have inferior performance compared to other learning methods. In Tables~\ref{tab:brodatz_results_clustering},~\ref{tab:rgbd_results_clustering},~\ref{tab:ethz_results_clustering}, and \ref{tab:yt_results_clustering}, a similar trend is seen in the classification setting.


\section{Conclusions}
\label{sec:conclude}
In this paper, we proposed a novel setup for dictionary learning and sparse coding of data in the form of SPD matrices. In contrast to prior methods that use proxy similarity distance measures to define the sparse coding approximation loss, our formulation used a loss driven by the natural Riemannian metric (affine invariant Riemannian metric) on the SPD manifold. We proposed an efficient adaptation of the well-known non-linear conjugate gradient method for learning the dictionary in the product space of SPD manifolds and a fast algorithm for sparse coding based on the spectral projected gradient. Our experiments on simulated and several benchmark computer vision datasets demonstrated the superior performance of our method against prior works; especially our results showed that learning the dictionary using our scheme leads to significantly better accuracy (in retrieval and classification) than other heuristic and approximate schemes to generate the dictionary.
\appendix
\label{app:proof_of_theorem_opt}
Here we prove Theorem~\ref{thm:opt}.
\begin{lemma}
  \label{lem:psd}
  Let $Z \in \genlinear(d)$ and let $X \in \spd{d}$. Then, $Z^TXZ \in \spd{d}$.
\end{lemma}

\begin{lemma}
  The Fr\'echet derivative~\cite[see e.g., Ch.~1]{higham08} of the map $X \mapsto \log X$ at a point $Z$ in the  direction $E$ is given by
  \small{
  \begin{equation}
    \label{eq:8}
    D\log(Z)(E) = \sint_0^1(\beta Z + (1-\beta)I)^{-1}E(\beta Z + (1-\beta)I)^{-1}d\beta.
  \end{equation}
  }
\end{lemma}
\begin{proof}
 See e.g.,~\cite[Ch.~11]{higham08}.
\end{proof}
\begin{corollary}
  \label{cor:log}
  Consider the map $\ell(\alpha) := \alpha \in \reals{n}_+ \mapsto \trace(\log(SM(\alpha)S)H)$, where $M$ is a map from $\reals{n}_+ \to \spd{d}$ and $H \in \sym{d}$, $S \in \spd{d}$. Then, for $1\le p \le n$, we have
  \begin{equation*}
    \tfrac{\partial\ell(\alpha)}{\partial \alpha_p} = \sint_0^1\trace[K_\beta S\tfrac{\partial M(\alpha)}{\partial \alpha_p}SK_\beta H]d\beta,
  \end{equation*}
  where $K_\beta := (\beta SM(\alpha)S + (1-\beta)I)^{-1}$.
\end{corollary}
\begin{proof}
  Simply apply the chain-rule of calculus and use linearity of $\trace(\cdot)$.
\end{proof}
\relax
\begin{lemma}
  \label{lem:inv}
  The Fr\'echet derivative of the map $X \mapsto X^{-1}$ at a point $Z$ in direction $E$ is given by
  \begin{equation}  
    D(Z^{-1})(E)=-Z^{-1}EZ^{-1}.
    \label{eq:7}
  \end{equation}
\end{lemma}

We are now ready to prove Theorem~\ref{thm:opt}.
\begin{proof}[Thm.~\ref{thm:opt}]
We show that the Hessian $\nabla^2\phi(\alpha)\succeq 0$  on $\mathcal A$. To ease presentation, we write $S=X^{-1/2}$,  $M\equiv M(\alpha)=\sum_i\alpha_iB_i$, and let $D_q$ denote the differential operator $D_{\alpha_q}$. Applying this operator to the first-derivative given by Lemma~\ref{lem:grad} (in Section~\ref{sec:sparse_coding_subproblem}), we obtain (using the product rule) the sum
  \begin{equation*}
    \begin{split}
      \trace & \bigl([D_q\log(SMS)] (SMS)^{-1}SB_pS\bigr) \\
      	& + \trace\bigl(\log(SMS)D_q[(SMS)^{-1}SB_pS]\bigr).
    \end{split}
  \end{equation*}
  We now treat these two terms individually. To the first we apply Corr.~\ref{cor:log}. So
  \begin{equation*}
    \begin{split}
      \trace & \bigl([D_q\log(SMS)] (SMS)^{-1}SB_pS\bigr) \\
      &= \sint_0^1\trace(K_\beta{}SB_qSK_\beta{}(SMS)^{-1}SB_pS)d\beta\\
      &=\sint_0^1\trace(SB_q SK_{\beta}(SMS)^{-1}SB_pSK_{\beta}\cdot)d\beta\\
      &=\sint_0^1\ip{\Psi_\beta(p)}{\Psi_\beta(q)}_M d\beta,
    \end{split}
  \end{equation*}
  where the inner-product $\ip{\cdot}{\cdot}_M$ is weighted by $(SMS)^{-1}$ and the map $\Psi_\beta(p) := SB_pSK_{\beta}$. We find a similar inner-product representation for the second term too. Starting with Lemma~\ref{lem:inv} and simplifying, we obtain
  \begin{equation*}
    \begin{split}
      \trace & \bigl(\log(SMS)D_q[(SMS)^{-1}SB_pS]\bigr) \\
      &= -\trace\bigl(\log(SMS)(SMS)^{-1}SB_qM^{-1}B_pS\bigr)\\
      &=\trace\bigl(-S\log(SMS)S^{-1}M^{-1}B_qM^{-1}B_p\bigr)\\
      &=\trace\bigl(M^{-1}B_p[-S\log(SMS)S^{-1}]M^{-1}B_q\bigr).
    \end{split}
  \end{equation*}
  By assumption $\sum_i \alpha_i B_i = M \preceq X$, which implies $SMS \preceq I$. Since $\log(\cdot)$ is operator monotone~\cite{bhatia07}, it follows that $\log(SMS) \preceq 0$; an application of Lemma~\ref{lem:psd} then yields $S\log(SMS)S^{-1} \preceq 0$. Thus, we obtain the weighted inner-product
  \begin{align*}
    \trace & \bigl(M^{-1}B_p[-S\log(SMS)S^{-1}]M^{-1}B_q\bigr) \\
    	&= \ip{M^{-1}B_p}{M^{-1}B_q}_L,
  \end{align*}
  where $L = [-S\log(SMS)S^{-1}] \succeq 0$, whereby $\ip{\cdot}{\cdot}_L$ is a valid inner-product.

  Thus, the second partial derivatives of $\phi$ may be ultimately written as
  \begin{equation*}
    \frac{\partial^2\phi(\alpha)}{\partial\alpha_p\partial\alpha_q} = \ip{\Gamma(B_q)}{\Gamma(B_p)},
  \end{equation*}
  for some map $\Gamma$ and some corresponding inner-product (the map and the inner-product are defined by our analysis above). Thus, we have established that the Hessian is a Gram matrix, which shows it is semidefinite. Moreover, if the $B_i$ are different ($1\le i\le n$), then the Hessian is strictly positive definite.
\end{proof}




{
\bibliographystyle{IEEEtran}
\bibliography{geosp}

\begin{thebibliography}{10}
\providecommand{\url}[1]{#1}
\csname url@samestyle\endcsname
\providecommand{\newblock}{\relax}
\providecommand{\bibinfo}[2]{#2}
\providecommand{\BIBentrySTDinterwordspacing}{\spaceskip=0pt\relax}
\providecommand{\BIBentryALTinterwordstretchfactor}{4}
\providecommand{\BIBentryALTinterwordspacing}{\spaceskip=\fontdimen2\font plus
\BIBentryALTinterwordstretchfactor\fontdimen3\font minus
  \fontdimen4\font\relax}
\providecommand{\BIBforeignlanguage}[2]{{%
\expandafter\ifx\csname l@#1\endcsname\relax
\typeout{** WARNING: IEEEtran.bst: No hyphenation pattern has been}%
\typeout{** loaded for the language `#1'. Using the pattern for}%
\typeout{** the default language instead.}%
\else
\language=\csname l@#1\endcsname
\fi
#2}}
\providecommand{\BIBdecl}{\relax}
\BIBdecl

\bibitem{porikli1}
{O.~Tuzel, F.~Porikli,~and~P.~Meer.}, ``{Region Covariance: A Fast Descriptor
  for Detection and Classification},'' in \emph{ECCV}, 2006.

\bibitem{cvpr_hartley}
S.~Jayasumana, R.~Hartley, M.~Salzmann, H.~Li, and M.~Harandi, ``Kernel methods
  on riemannian manifolds with gaussian rbf kernels,'' \emph{IEEE Transactions
  on Pattern Analysis and Machine Intelligence}, 2015.

\bibitem{porikli3}
{F.~Porikli, and O.~Tuzel}, ``Covariance tracker,'' \emph{Computer Vision and
  Pattern Recognition}, June 2006.

\bibitem{cherian2011dirichlet}
A.~Cherian, V.~Morellas, N.~Papanikolopoulos, and S.~J. Bedros, ``Dirichlet
  process mixture models on symmetric positive definite matrices for appearance
  clustering in video surveillance applications,'' in \emph{Computer Vision and
  Pattern Recognition}.\hskip 1em plus 0.5em minus 0.4em\relax IEEE, 2011.

\bibitem{fehr2012compact}
D.~Fehr, A.~Cherian, R.~Sivalingam, S.~Nickolay, V.~Morellas, and
  N.~Papanikolopoulos, ``Compact covariance descriptors in {3D} point clouds
  for object recognition,'' in \emph{International conference on Robotics and
  Automation}.\hskip 1em plus 0.5em minus 0.4em\relax IEEE, 2012.

\bibitem{ma2012affine}
B.~Ma, Y.~Wu, and F.~Sun, ``Affine object tracking using kernel-based region
  covariance descriptors,'' in \emph{Foundations of Intelligent Systems}.\hskip
  1em plus 0.5em minus 0.4em\relax Springer, 2012, pp. 613--623.

\bibitem{mairalbook}
J.~Mairal, F.~Bach, and J.~Ponce, ``Sparse modeling for image and vision
  processing,'' \emph{Foundations and Trends in Computer Graphics and Vision},
  vol.~8, no.~2, pp. 85--283, 2014.

\bibitem{guha2012learning}
T.~Guha and R.~K. Ward, ``Learning sparse representations for human action
  recognition,'' \emph{IEEE Transactions on Pattern Analysis and Machine
  Intelligence}, vol.~34, no.~8, pp. 1576--1588, 2012.

\bibitem{wright2009robust}
J.~Wright, A.~Y. Yang, A.~Ganesh, S.~S. Sastry, and Y.~Ma, ``Robust face
  recognition via sparse representation,'' \emph{IEEE Transactions on Pattern
  Analysis and Machine Intelligence}, vol.~31, no.~2, pp. 210--227, 2009.

\bibitem{cherian2014nearest}
A.~Cherian, ``Nearest neighbors using compact sparse codes,'' in
  \emph{International Conference on Machine Learning}, 2014.

\bibitem{harandi2013dictionary}
M.~Harandi, C.~Sanderson, C.~Shen, and B.~C. Lovell, ``Dictionary learning and
  sparse coding on grassmann manifolds: An extrinsic solution,'' in
  \emph{International Conference on Computer Vision}.\hskip 1em plus 0.5em
  minus 0.4em\relax IEEE, 2013.

\bibitem{harandi2012sparse}
M.~T. Harandi, C.~Sanderson, R.~Hartley, and B.~C. Lovell, ``Sparse coding and
  dictionary learning for symmetric positive definite matrices: A kernel
  approach,'' in \emph{European Conference on Computer Vision}.\hskip 1em plus
  0.5em minus 0.4em\relax Springer, 2012.

\bibitem{lilog}
P.~Li, Q.~Wang, W.~Zuo, and L.~Zhang, ``Log-euclidean kernels for sparse
  representation and dictionary learning,'' in \emph{ICCV}.\hskip 1em plus
  0.5em minus 0.4em\relax IEEE, 2013.

\bibitem{sivalingam2010tensor}
R.~Sivalingam, D.~Boley, V.~Morellas, and N.~Papanikolopoulos, ``Tensor sparse
  coding for region covariances,'' in \emph{ECCV}.\hskip 1em plus 0.5em minus
  0.4em\relax Springer, 2010.

\bibitem{sra2011generalized}
S.~Sra and A.~Cherian, ``Generalized dictionary learning for symmetric positive
  definite matrices with application to nearest neighbor retrieval,'' in
  \emph{European Conference on Machine Learning}.\hskip 1em plus 0.5em minus
  0.4em\relax Springer, 2011.

\bibitem{arsigny2006log}
V.~Arsigny, P.~Fillard, X.~Pennec, and N.~Ayache, ``{Log-Euclidean metrics for
  fast and simple calculus on diffusion tensors},'' \emph{Magnetic Resonance in
  Medicine}, vol.~56, no.~2, pp. 411--421, 2006.

\bibitem{sra2011positive}
S.~Sra, ``Positive definite matrices and the {S}-divergence,'' \emph{arXiv
  preprint arXiv:1110.1773}, 2011.

\bibitem{moakher2006symmetric}
M.~Moakher and P.~G. Batchelor, ``Symmetric positive-definite matrices: From
  geometry to applications and visualization,'' in \emph{Visualization and
  Processing of Tensor Fields}.\hskip 1em plus 0.5em minus 0.4em\relax
  Springer, 2006, pp. 285--298.

\bibitem{pennec2006riemannian}
X.~Pennec, P.~Fillard, and N.~Ayache, ``{A Riemannian framework for tensor
  computing},'' \emph{International Journal of Computer Vision}, vol.~66,
  no.~1, pp. 41--66, 2006.

\bibitem{rothaus1960domains}
O.~S. Rothaus, ``Domains of positivity,'' in \emph{Abhandlungen aus dem
  Mathematischen Seminar der Universit{\"a}t Hamburg}, vol.~24, no.~1.\hskip
  1em plus 0.5em minus 0.4em\relax Springer, 1960, pp. 189--235.

\bibitem{cherian2014riemannian}
A.~Cherian and S.~Sra, ``Riemannian sparse coding for positive definite
  matrices,'' in \emph{European Conference on Computer Vision}.\hskip 1em plus
  0.5em minus 0.4em\relax Springer, 2014.

\bibitem{absil2009optimization}
P.-A. Absil, R.~Mahony, and R.~Sepulchre, \emph{Optimization algorithms on
  matrix manifolds}.\hskip 1em plus 0.5em minus 0.4em\relax Princeton
  University Press, 2009.

\bibitem{cheng2013novel}
G.~Cheng and B.~C. Vemuri, ``A novel dynamic system in the space of spd
  matrices with applications to appearance tracking,'' \emph{SIAM journal on
  imaging sciences}, vol.~6, no.~1, pp. 592--615, 2013.

\bibitem{nesterov2002riemannian}
Y.~E. Nesterov and M.~J. Todd, ``{On the Riemannian geometry defined by
  self-concordant barriers and interior-point methods},'' \emph{Foundations of
  Computational Mathematics}, vol.~2, no.~4, pp. 333--361, 2002.

\bibitem{lang2012fundamentals}
S.~Lang, \emph{Fundamentals of differential geometry}.\hskip 1em plus 0.5em
  minus 0.4em\relax Springer Science \& Business Media, 2012, vol. 191.

\bibitem{bhatia07}
R.~Bhatia, \emph{{Positive Definite Matrices}}.\hskip 1em plus 0.5em minus
  0.4em\relax Princeton University Press, 2007.

\bibitem{aharon2006img}
M.~Aharon, M.~Elad, and A.~Bruckstein, ``{K-SVD}: An algorithm for designing
  overcomplete dictionaries for sparse representation,'' \emph{IEEE
  Transactions on Signal Processing}, vol.~54, no.~11, pp. 4311--4322, 2006.

\bibitem{denoising2006}
M.~Elad and M.~Aharon, ``Image denoising via learned dictionaries and sparse
  representation,'' in \emph{Computer Vision and Pattern Recognition}, 2006.

\bibitem{sivalingam2011positive}
R.~Sivalingam, D.~Boley, V.~Morellas, and N.~Papanikolopoulos, ``Positive
  definite dictionary learning for region covariances,'' in \emph{International
  Conference on Computer Vision}.\hskip 1em plus 0.5em minus 0.4em\relax IEEE,
  2011.

\bibitem{guo2010action}
K.~Guo, P.~Ishwar, and J.~Konrad, ``Action recognition using sparse
  representation on covariance manifolds of optical flow,'' in
  \emph{International Conference on Advanced Video and Signal based
  Surveillance}.\hskip 1em plus 0.5em minus 0.4em\relax IEEE, 2010.

\bibitem{ho2013nonlinear}
J.~Ho, Y.~Xie, and B.~Vemuri, ``On a nonlinear generalization of sparse coding
  and dictionary learning,'' in \emph{International Conference on Machine
  Learning}, 2013.

\bibitem{cichocki2015log}
A.~Cichocki, S.~Cruces, and S.-i. Amari, ``Log-determinant divergences
  revisited: Alpha-beta and gamma log-det divergences,'' \emph{Entropy},
  vol.~17, no.~5, pp. 2988--3034, 2015.

\bibitem{cherian2013jensen}
A.~Cherian, S.~Sra, A.~Banerjee, and N.~Papanikolopoulos, ``Jensen-bregman
  logdet divergence with application to efficient similarity search for
  covariance matrices,'' \emph{IEEE Transactions on Pattern Analysis and
  Machine Intelligence}, vol.~35, no.~9, pp. 2161--2174, 2013.

\bibitem{harandi2009riemannian}
M.~Harandi and M.~Salzmann, ``Riemannian coding and dictionary learning:
  Kernels to the rescue,'' \emph{Computer Vision and Pattern Recognition},
  2015.

\bibitem{feragen2014geodesic}
A.~Feragen, F.~Lauze, and S.~Hauberg, ``Geodesic exponential kernels: When
  curvature and linearity conflict,'' \emph{arXiv preprint arXiv:1411.0296},
  2014.

\bibitem{pennec2005riemannian}
X.~Pennec, R.~Stefanescu, V.~Arsigny, P.~Fillard, and N.~Ayache, ``Riemannian
  elasticity: A statistical regularization framework for non-linear
  registration,'' in \emph{International Conference on Medical Image Computing
  and Computer Assisted Interventions}.\hskip 1em plus 0.5em minus 0.4em\relax
  Springer, 2005.

\bibitem{absil2007trust}
P.-A. Absil, C.~G. Baker, and K.~A. Gallivan, ``Trust-region methods on
  riemannian manifolds,'' \emph{Foundations of Computational Mathematics},
  vol.~7, no.~3, pp. 303--330, 2007.

\bibitem{bertsekas99}
D.~P. Bertsekas and D.~P. Bertsekas, \emph{{Nonlinear Programming}},
  2nd~ed.\hskip 1em plus 0.5em minus 0.4em\relax Athena Scientific, 1999.

\bibitem{urruty}
J.-B. Hiriart-Urruty and C.~Lemar\'echal, \emph{{Fundamentals of convex
  analysis}}.\hskip 1em plus 0.5em minus 0.4em\relax Springer, 2001.

\bibitem{bb88}
J.~Barzilai and J.~M. Borwein, ``{Two-Point Step Size Gradient Methods},''
  \emph{IMA J. Num. Analy.}, vol.~8, no.~1, 1988.

\bibitem{schmidt09}
M.~Schmidt, E.~van~den Berg, M.~Friedlander, and K.~Murphy, ``{Optimizing
  Costly Functions with Simple Constraints: A Limited-Memory Projected
  Quasi-Newton Algorithm},'' in \emph{International Conference on Artificial
  Intelligence and Statistics}, 2009.

\bibitem{birgin01}
E.~G. Birgin, J.~M. Mart{\'i}nez, and M.~Raydan, ``{Algorithm 813: SPG -
  Software for Convex-constrained Optimization},'' \emph{ACM Transactions on
  Mathematical Software}, vol.~27, pp. 340--349, 2001.

\bibitem{pang2008gabor}
Y.~Pang, Y.~Yuan, and X.~Li, ``{Gabor-based region covariance matrices for face
  recognition},'' \emph{IEEE Transactions on Circuits and Systems for Video
  Technology}, vol.~18, no.~7, pp. 989--993, 2008.

\bibitem{harandi2014manifold}
M.~T. Harandi, M.~Salzmann, and R.~Hartley, ``From manifold to manifold:
  Geometry-aware dimensionality reduction for spd matrices,'' in \emph{European
  Conference on Computer Vision}.\hskip 1em plus 0.5em minus 0.4em\relax
  Springer, 2014.

\bibitem{lai2011large}
K.~Lai, L.~Bo, X.~Ren, and D.~Fox, ``A large-scale hierarchical multi-view
  {RGB-D} object dataset,'' in \emph{International Conference on Robotics and
  Automation}, 2011.

\bibitem{ojala1996comparative}
T.~Ojala, M.~Pietik{\"a}inen, and D.~Harwood, ``A comparative study of texture
  measures with classification based on featured distributions,'' \emph{Pattern
  recognition}, vol.~29, no.~1, pp. 51--59, 1996.

\bibitem{ess2007depth}
A.~Ess, B.~Leibe, and L.~V. Gool, ``Depth and appearance for mobile scene
  analysis,'' in \emph{International Conference on Computer Vision}.\hskip 1em
  plus 0.5em minus 0.4em\relax IEEE, 2007.

\bibitem{youtubefaces}
T.~H. Lior~Wolf and I.~Maoz, ``Face recognition in unconstrained videos with
  matched background similarity,'' in \emph{Computer Vision and Pattern
  Recognition}.\hskip 1em plus 0.5em minus 0.4em\relax IEEE, 2011.

\bibitem{de2008texture}
R.~Luis-Garc{\'\i}a, R.~Deriche, and C.~Alberola-L{\'o}pez, ``Texture and color
  segmentation based on the combined use of the structure tensor and the image
  components,'' \emph{Signal Processing}, vol.~88, no.~4, pp. 776--795, 2008.

\bibitem{ma2012bicov}
B.~Ma, Y.~Su, F.~Jurie \emph{et~al.}, ``Bicov: a novel image representation for
  person re-identification and face verification,'' in \emph{BMVC}, 2012.

\bibitem{schwartz09d}
W.~Schwartz and L.~Davis, ``{Learning Discriminative Appearance-Based Models
  Using Partial Least Squares},'' in \emph{Proceedings of the XXII Brazilian
  Symposium on Computer Graphics and Image Processing}, 2009.

\bibitem{sivalingam2015tensor}
R.~Sivalingam, D.~Boley, V.~Morellas, and N.~Papanikolopoulos, ``Tensor
  dictionary learning for positive definite matrices,'' \emph{IEEE Transactions
  on Image Processing}, vol.~24, no.~11, pp. 4592--4601, 2015.

\bibitem{fehr2013covariance}
D.~A. Fehr, \emph{Covariance based point cloud descriptors for object detection
  and classification}.\hskip 1em plus 0.5em minus 0.4em\relax University Of
  Minnesota, 2013.

\bibitem{zhu2012face}
X.~Zhu and D.~Ramanan, ``Face detection, pose estimation, and landmark
  localization in the wild,'' in \emph{Computer Vision and Pattern
  Recognition}.\hskip 1em plus 0.5em minus 0.4em\relax IEEE, 2012.

\bibitem{boumal2014manopt}
N.~Boumal, B.~Mishra, P.-A. Absil, and R.~Sepulchre, ``Manopt, a matlab toolbox
  for optimization on manifolds,'' \emph{The Journal of Machine Learning
  Research}, vol.~15, no.~1, pp. 1455--1459, 2014.

\bibitem{higham08}
N.~Higham, \emph{{Functions of Matrices: Theory and Computation}}.\hskip 1em
  plus 0.5em minus 0.4em\relax SIAM, 2008.

\end{thebibliography}
}
\end{document}